%% file: main.tex
\documentclass[10pt,twocolumn,letterpaper]{article}

\usepackage[pagenumbers]{iccv} 

\input{preamble}

\usepackage{graphicx}
\usepackage{amsmath}
\usepackage{amsthm}
\usepackage{capt-of}
\usepackage{wrapfig}
\usepackage{multirow}
\usepackage{comment}

\usepackage{algorithm}
\usepackage{algpseudocode}

\theoremstyle{definition}
\newtheorem{theorem}{Theorem}[section]

\newtheorem{lemma}[theorem]{Lemma}

\newtheorem{definition}[theorem]{Definition}

\definecolor{iccvblue}{rgb}{0.21,0.49,0.74}
\usepackage[pagebackref,breaklinks,colorlinks,allcolors=iccvblue]{hyperref}

\def\paperID{6574} 
\def\confName{ICCV}
\def\confYear{2025}

\title{DMesh++: An Efficient Differentiable Mesh for Complex Shapes}

\author{Sanghyun Son$^{*}$\\
University of Maryland\\
{\tt\small shh1295@umd.edu}
\and
Matheus Gadelha\\
Adobe Research\\
{\tt\small gadelha@adobe.com}
\and
Yang Zhou\\
Adobe Research\\
{\tt\small yazhou@adobe.com}
\and
Matthew Fisher\\
Adobe Research\\
{\tt\small matfishe@adobe.com}
\and
Zexiang Xu\\
Adobe Research\\
{\tt\small zexu@adobe.com}
\and
Yi-Ling Qiao\\
University of Maryland\\
{\tt\small yilingq@umd.edu}
\and
Ming C. Lin\\
University of Maryland\\
{\tt\small lin@umd.edu}
\and
Yi Zhou\\
Adobe Research\\
{\tt\small yizho@adobe.com}
}

\begin{document}
\input{sec/0_abstract}

\input{sec/1_intro}
\input{sec/2_related}
\input{sec/3_formulation}
\input{sec/4_recon}

\input{sec/5_experiments}
\input{sec/6_conclusion}
{
    \small
    \bibliographystyle{ieeenat_fullname}
    \bibliography{main}
}

\input{sec/X_suppl}

\end{document}

%% file: preamble.tex
\newcommand{\red}[1]{{\color{red}#1}}
\newcommand{\blue}[1]{{\color{blue}#1}}

\definecolor{darkgreen}{rgb}{0.0, 0.5, 0.0}
\definecolor{outsidepink}{rgb}{0.9, 0.3, 0.3}
\definecolor{insideblue}{rgb}{0.4, 0.5, 0.9}

%% file: sec/0_abstract.tex
\twocolumn[{%
\renewcommand\twocolumn[1][]{#1}%
\maketitle
\begin{center}
    \centering
    \captionsetup{type=figure}
    \includegraphics[width=0.95\textwidth]{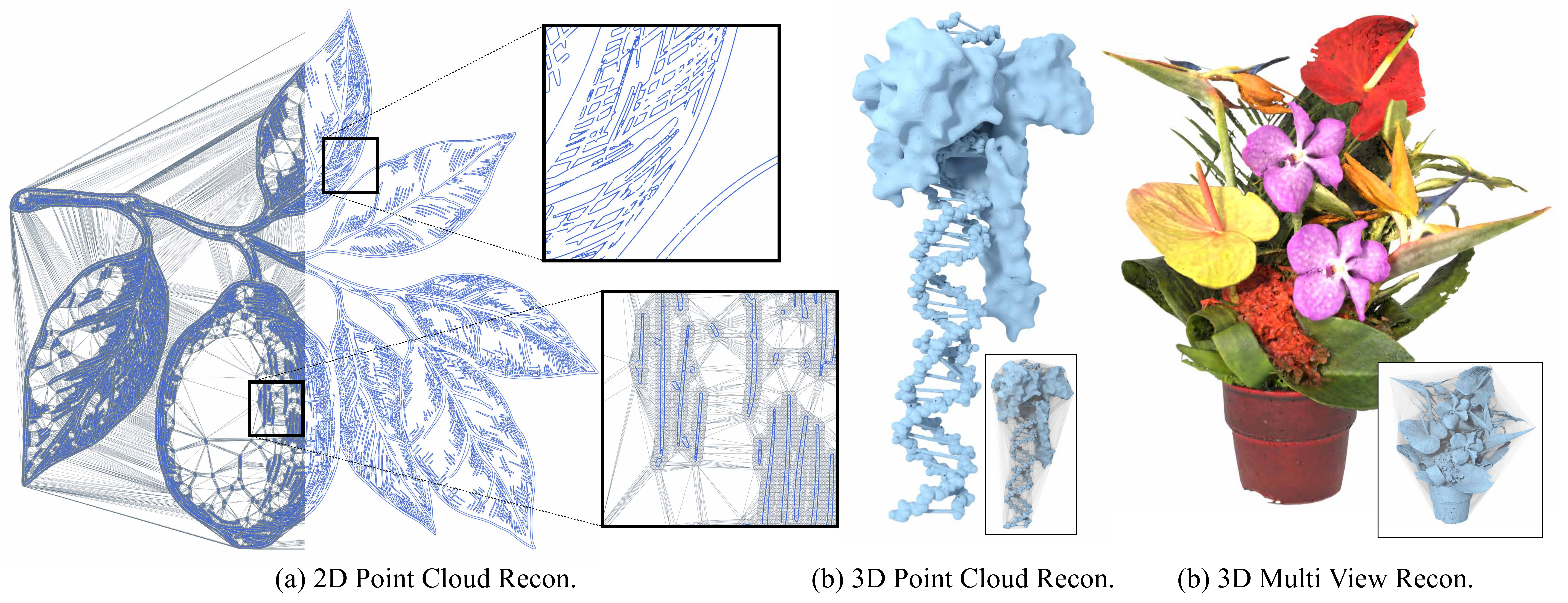}
    \captionof{figure}{\textbf{DMesh++ for complex 2D and 3D shapes.} DMesh++ encodes all geometric and topological information into continuous point features. (a) By optimizing these point features, DMesh++ is able to reconstruct complex 2D drawings from sample points. (b) This approach is also applicable to 3D, where it reconstructs the complex geometric structure of DNA from a point cloud. (c) By incorporating additional color features, DMesh++ can reconstruct complex, colored 3D shapes from multi-view images. For each result, the \textit{“imaginary”} part is rendered in gray, while the \textit{“real”} part—which defines the final mesh—is rendered in other colors. (\cref{sec:preliminary}).}
    \label{fig:teaser}
\end{center}%
}]

\begin{abstract}
Recent probabilistic methods for 3D triangular meshes capture diverse shapes by differentiable mesh connectivity, but face high computational costs with increased shape details. We introduce a new differentiable mesh processing method that addresses this challenge and efficiently handles meshes with intricate structures. Our method reduces time complexity from $O(N)$ to $O(\log N)$ and requires significantly less memory than previous approaches. Building on this innovation, we present a reconstruction algorithm capable of generating complex 2D and 3D shapes from point clouds or multi-view images. Visit our \href{https://sonsang.github.io/dmesh2-project}{project page} for source code and supplementary material.
\end{abstract}

\vspace{-1em}
{\scriptsize               
\setlength{\parskip}{0pt}  
\noindent                  
* This work was mainly done during internship at Adobe Research and continued as a collaborative effort with UMD.\\
** The paper was last modified on Jul.\,6, 2025.\par
}
\normalsize

%% file: sec/1_intro.tex
\section{Introduction}
\label{sec:intro}

Among various possible shape representations, a mesh is often favored for a wide range of downstream tasks due to its efficiency, versatility, and controllability. A mesh is defined by its vertices' position and their connectivity in the form of edges and faces. This connectivity is discrete in nature, and also the number of possible connectivities grows exponentially with the number of points, which prevents meshes from being differentiable shape representations (\cref{fig:mesh_dmesh}). To address this, recent data-driven efforts have attempted to predict mesh connectivity using Transformer-based models~\citep{siddiqui2024meshgpt, chen2024meshanything, chen2024meshanything2, shen2024spacemesh}. However, these methods face inherent challenges with robustness to outlier meshes, potential self-intersections, and high computational costs.

On another route, Son et al.~\citep{son2024dmesh} introduced a new form of differentiable mesh called DMesh, which is essentially a probabilistic approach. For a given set of points, they explicitly compute probabilities for possible face combinations to exist on the mesh based on the continuous point-wise features. This approach minimizes several mesh degeneracies, and is free from outliers, as it is not data-driven (\cref{fig:mesh_dmesh}). Therefore, this probabilistic approach opens up a new venue to adopt meshes in a machine learning pipeline, such as generative models~\citep{wei2024meshlrm, zhang2024clay}. However, it suffers from excessive computational cost when the number of points increases (\cref{fig:minball-speed}), which limits its applicability for representing complex shapes with detailed structures.

In this work, we introduce DMesh++, which overcomes the computational limitations of DMesh while retaining its core advantages. To that end, we present \emph{Minimum-Ball} algorithm. While the computational cost to evaluate face probability is $O(N)$ for DMesh, where $N$ is the number of points that define the mesh, our \emph{Minimum-Ball} algorithm has $O(\log N)$ computational cost (\cref{sec:min-ball,fig:minball-speed}). 

The direct application of DMesh++ is a reconstruction task. It effectively reconstructs complex 2D and 3D meshes of diverse topology from point clouds or multi-view images (\cref{fig:teaser,fig:3d-pc-qual,fig:mv-qual}). During the optimization of continuous point-wise features, we observe dynamic topological changes in the mesh that recover the target shape (\cref{fig:tsne}).

\begin{figure}
    \centering
    \includegraphics[width=1.0\linewidth]{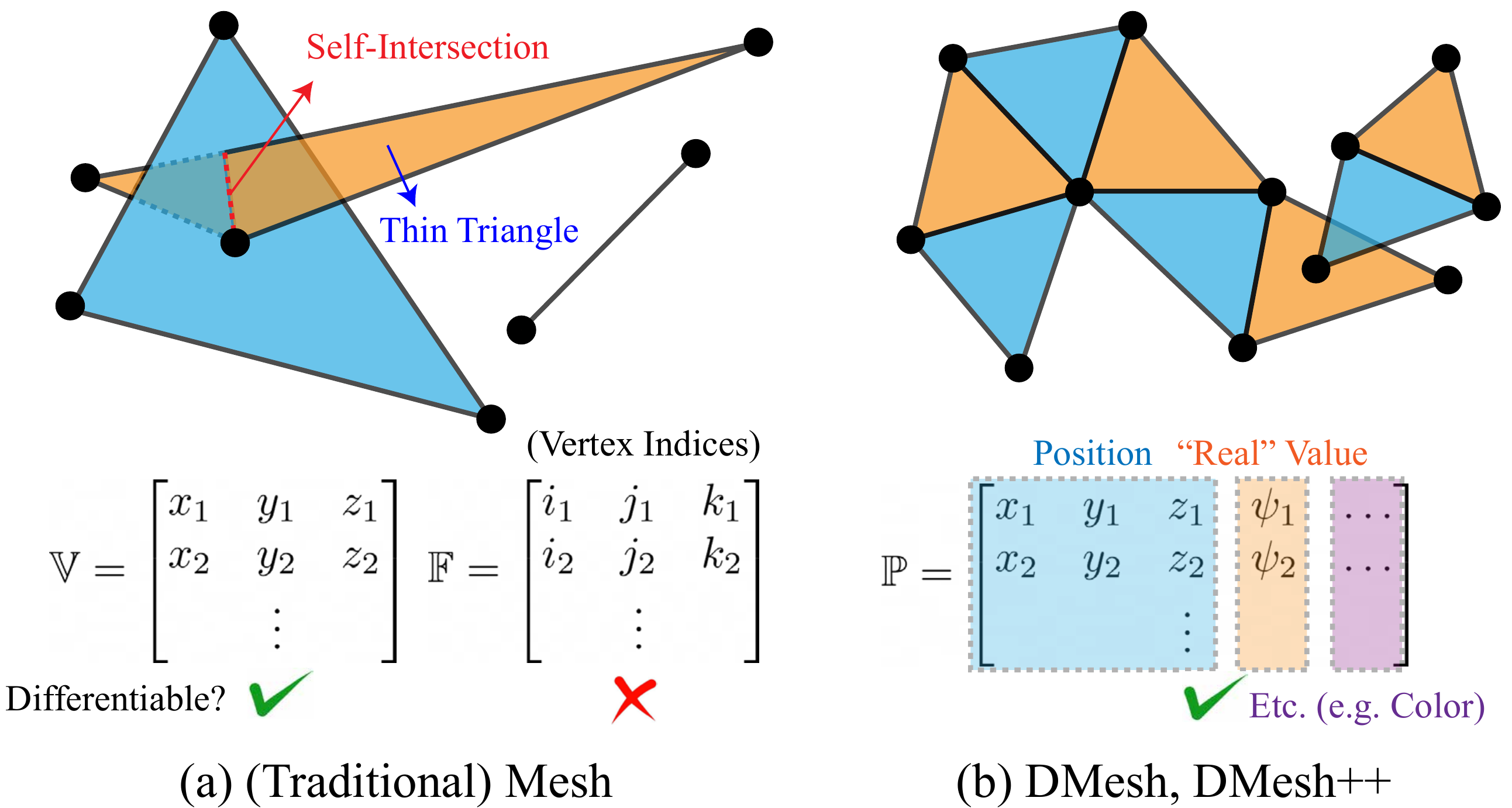}
    \caption{\textbf{Conceptual comparison of traditional mesh and variants of DMesh~\citep{son2024dmesh}.} Traditional meshes employ a non-differentiable, discrete data structure, $\mathbb{F}$, to store vertex indices that define connectivity, whereas DMesh++ encodes connectivity and additional information into continuous point-wise features, $\mathbb{P}$. Mesh generated from DMesh++ avoids several degeneracies—such as self-intersections and thin triangles—that can compromise its suitability for downstream applications.}
    \label{fig:mesh_dmesh}
    \vspace{-1em}
\end{figure}

To summarize, our contributions are the following:
\begin{itemize} 
    \item We present DMesh++, an enhanced version of DMesh~\citep{son2024dmesh}, which overcomes its computational bottlenecks by employing the \textit{Minimum-Ball} algorithm.
    \item We propose a reconstruction algorithm that incorporates efficient loss formulations and additional mesh operations to effectively recover 2D and 3D shapes from point clouds or multi-view images.
    \item We validate our approach on 500+ mesh models with diverse topology, which are collected from Thingi10K~\citep{zhou2016thingi10k} and Objaverse~\citep{deitke2023objaverse} dataset. 
\end{itemize}

%% file: sec/2_related.tex
\section{Related Work}

While meshes offer an efficient and flexible representation of shapes, they are mainly constrained by their connectivity issues, which limit their applicability in machine learning. To address these challenges, shape inference in machine learning has evolved through three stages.

\vspace{4pt}
\noindent\textbf{Using Alternative Differentiable Shape Representations.}
Rather than handling mesh directly, some prior work extract mesh from alternative differentiable shape representations. Neural implicit representations, like (un)signed distance fields~\citep{park2019deepsdf, yariv2021volsdf, wang2021neus, wang2023neus2, Oechsle2021ICCV, wei2023neumanifold, liu2023neudf, long2023neuraludf, yu2023surf}, encode distance fields in neural networks, and use iso-surface extraction algorithms~\citep{lorensen1998marching, ju2002dual, guillard2022meshudf} to generate the final mesh. Another method encodes distance directly into spatial points and applies differentiable iso-surface extraction~\citep{liao2018deep, shen2021deep, shen2023flexible, wei2023neumanifold, munkberg2022extracting, liu2023ghost, mehta2022level}. While often more efficient, these methods typically cannot handle open surfaces; though \cite{liu2023ghost} does, it cannot represent non-orientable geometries. Gaussian Splatting~\citep{kerbl20233d} also encodes visual data as spatial “splats” but lacks the geometric accuracy of implicit functions.

\begin{figure}
    \centering
    \includegraphics[width=1.0\linewidth]{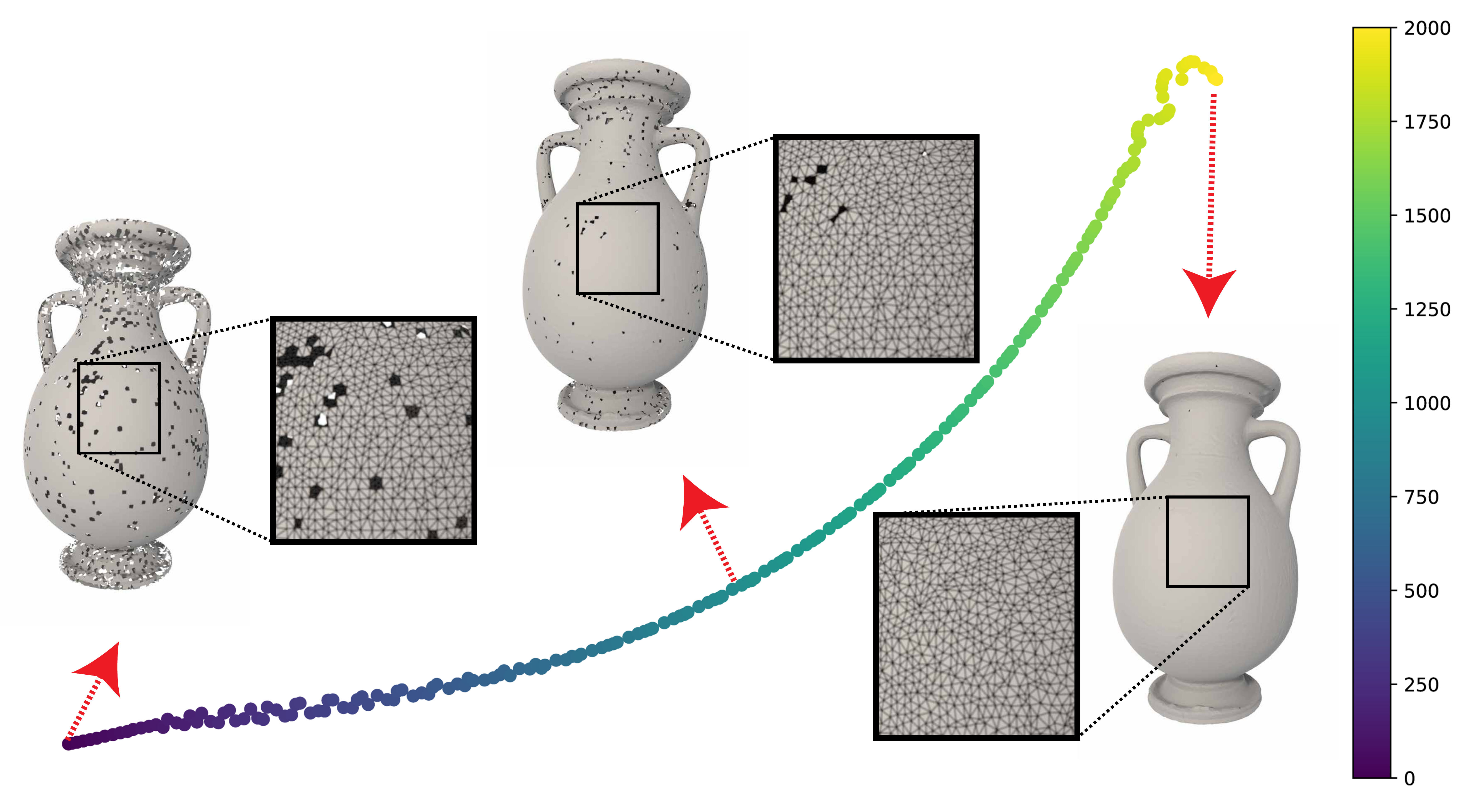}
    \caption{\textbf{Mesh topology change during 3D point cloud reconstruction of a vase.} We visualize the evolution of point-wise features by projecting them onto a 2-dimensional space using T-SNE~\citep{van2008visualizing} at each optimization step (ranging from 0 to 2K). As these features continuously evolve, the mesh undergoes discrete topological changes to progressively recover the target shape.}
    \label{fig:tsne}
    \vspace{-1em}
\end{figure}

\vspace{2pt}
\noindent\textbf{Inferring Meshes Differentiably.}
The main challenge in differentiable mesh handling is the exponential growth of possible vertex connections as vertex count increases. To simplify this challenge, most prior works assume fixed registration and permit only local connectivity changes~\citep{zhou2020fully, chen2019learning, nicolet2021large, liu2019soft, laine2020modular, palfinger2022continuous}. Recently, data-driven approaches have aimed to overcome these limitations by training generative models~\citep{siddiqui2024meshgpt, chen2024meshanything, chen2024meshanything2, shen2024spacemesh} that predict vertex connectivity from point clouds. Specifically, SpaceMesh~\citep{shen2024spacemesh} ensures combinatorial manifold mesh generation. However, these models struggle with outliers and self-intersections.


\vspace{4pt}
\noindent\textbf{Designing A Differentiable Form of Mesh.}
Son et al.~\citep{son2024dmesh} recently introduced DMesh, a differentiable mesh formulation using a probabilistic approach. DMesh augments each point with two continuous values, along with its position (\cref{fig:mesh_dmesh}), and applies a ``tessellation'' function in \cref{eq:tessellation-dmesh} to deterministically generate a mesh from a point set. This method adapts to various geometric topologies, including non-orientable open surfaces, and avoids self-intersections. With DMesh, optimizing or inferring only point-wise features is sufficient to generate the mesh.



However, DMesh's tessellation function is slow due to its reliance on Weighted Delaunay Triangulation (WDT), which has a practical time complexity of $O(N)$ for $N$ points using the CGAL package~\citep{cgal:pt-t3-23b}. For $N=100K$ in 3D, the runtime can reach up to 800 milliseconds (\cref{fig:minball-speed}), limiting DMesh's applicability for complex shapes requiring finer detail. Moreover, it is hardly possible to accelerate WDT using parallelization, because of its inherent race conditions. Therefore, in this work, we eliminate WDT, and propose a more efficient differentiable mesh formulation.


%% file: sec/3_formulation.tex
\section{Formulation}
In this section, we first provide the high-level formulation for computing probability of a face to exist in the mesh. Then, we introduce \emph{Minimum-Ball} (\cref{sec:min-ball}), which is our primary algorithm.

\begin{figure}[t]
    \centering
    \begin{subfigure}[b]{0.49\linewidth}
        \centering
        \includegraphics[width=\textwidth]{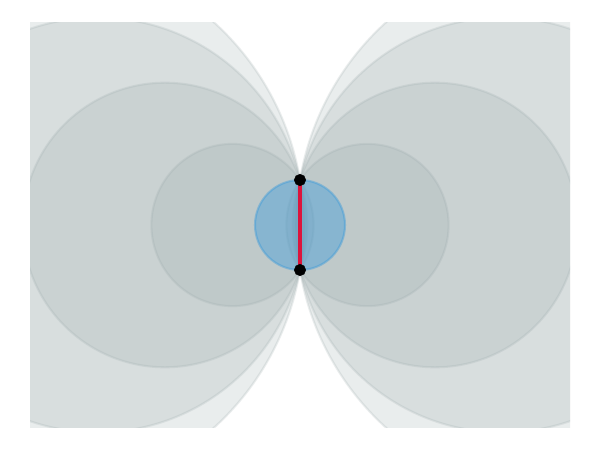}
    \end{subfigure}
    \hfill
    \begin{subfigure}[b]{0.49\linewidth}
        \centering
        \includegraphics[width=\textwidth]{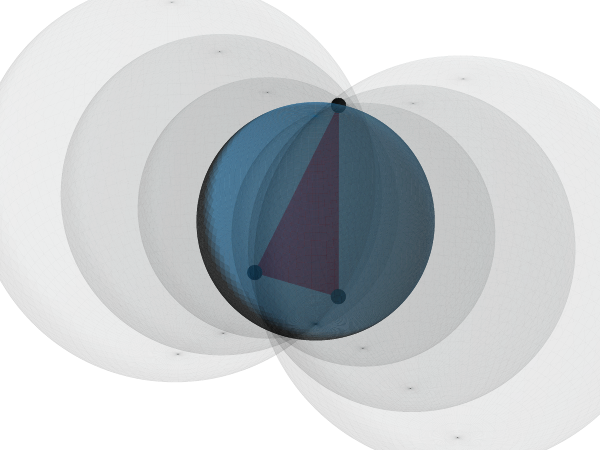}
    \end{subfigure}
    \caption{\textbf{Bounding balls of a face $F$ (red) in 2D (left) and 3D (right).} The minimum bounding ball ($B_F$) is rendered in blue, while the others are rendered in gray.}
    \label{fig:minball}
    \vspace{-1.5em}
\end{figure}

\subsection{Preliminary}
\label{sec:preliminary}

In this work, we refer to a $(d-1)$-simplex in $d$-dimensional space as a ``face'' (e.g., a line segment for $d=2$ or a triangle for $d=3$). DMesh~\citep{son2024dmesh} tessellates $d$-dimensional convex space using faces, with the actual surface on ``real'' faces and ``imaginary'' faces enclosing the ``real'' part to support the convex space (\cref{fig:teaser}).

In DMesh, each point is a $(d+2)$-dimensional vector: the first $d$ values denote position, while the remaining two represent the Weighted Delaunay Triangulation (WDT) weight ($w$)~\citep{aurenhammer1987power} and real value ($\psi$). The $\psi \in [0, 1]$ of a point indicates whether it lies on the shape, specifically if $\Psi(p) > 0.5$, where $\Psi(p)$ is $\psi$ of a point $p$.

For a point set $\mathbb{P}$, let $\mathbb{F}_{wdt}$ represent the faces in WDT of $\mathbb{P}$. DMesh then introduces a \textbf{``tessellation''} function to determine if a face $F$ exists on the mesh:
\begin{equation}
\label{eq:tessellation-dmesh}
    T_{DMesh}(\mathbb{P}, F) = (F \in \mathbb{F}_{wdt}) \wedge (\min_{p\in F}\Psi(p) > 0.5).
\end{equation}

DMesh++ introduces an alternative tessellation function for faster processing. By removing the need for WDT, we eliminate the WDT weight and represent each point as a $(d+1)$-dimensional vector, $(x_1, ..., x_d, \psi)$\footnote{Points could carry additional features, such as color (\cref{fig:mesh_dmesh}).}. In place of WDT, we implement a faster scheme called the \emph{Minimum-Ball} condition (Definition~\ref{def:min-ball}) for defining the tessellation function. Letting $\mathbb{F}_{min}$ represent the set of faces that meet this condition, we define the tessellation function as
\begin{equation}
\label{eq:tessellation_ours}
    T_{DMesh++}(\mathbb{P}, F) = (F \in \mathbb{F}_{min}) \wedge (\min_{p \in F}\Psi(p) > 0.5).
\end{equation}


In our differentiable framework, we compute probability of $F$ to satisfy these two conditions: $\Lambda_{min}$ and $\Lambda_{real}$, respectively. Then, we compute the final probability of $F$ to exist on the mesh as $\Lambda(F) = \Lambda_{min}(F) \times \Lambda_{real}(F)$. For $\Lambda_{real}$, we use differentiable $\min$ operator as DMesh. In the next section, we explain how we define $\Lambda_{min}$.

\begin{figure}[t]
    \centering
    \includegraphics[width=\linewidth]{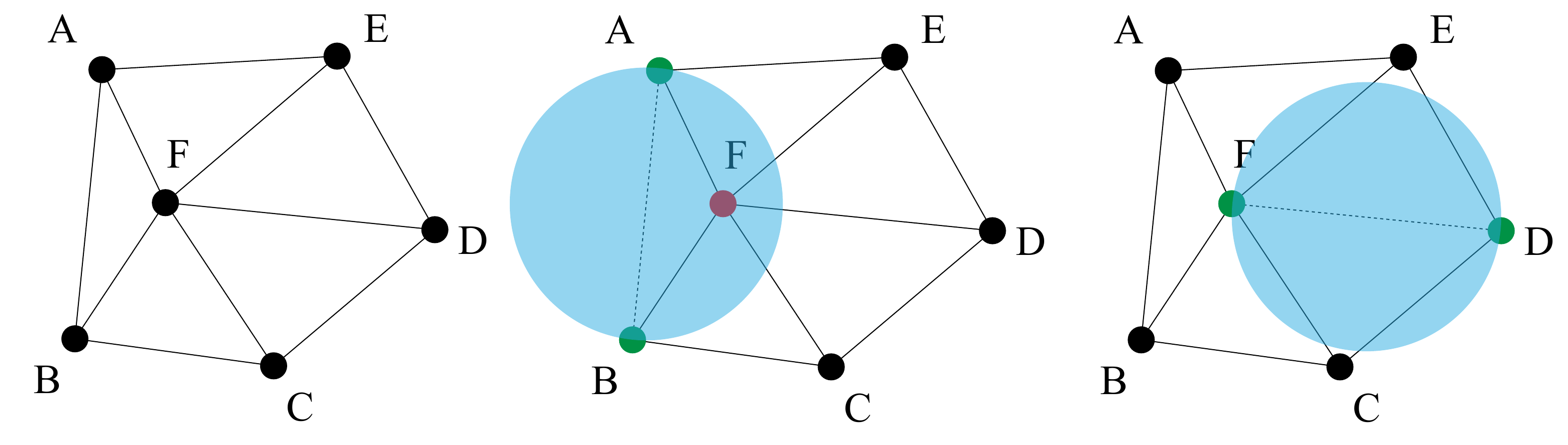}
    \caption{\textbf{Minimum-Ball condition in 2D.} In the left, 2D Delaunay Triangulation (DT) of 6 points is given. In middle and right figure, we render $B_{F}$ for two faces ($\overline{AB}$, $\overline{DF}$) in blue.}
    \label{fig:minball-point}
    \vspace{-1.5em}
\end{figure}

\subsection{Minimum-Ball Algorithm}
\label{sec:min-ball}

\begin{figure*}
    \centering
    \includegraphics[width=1.0\textwidth]{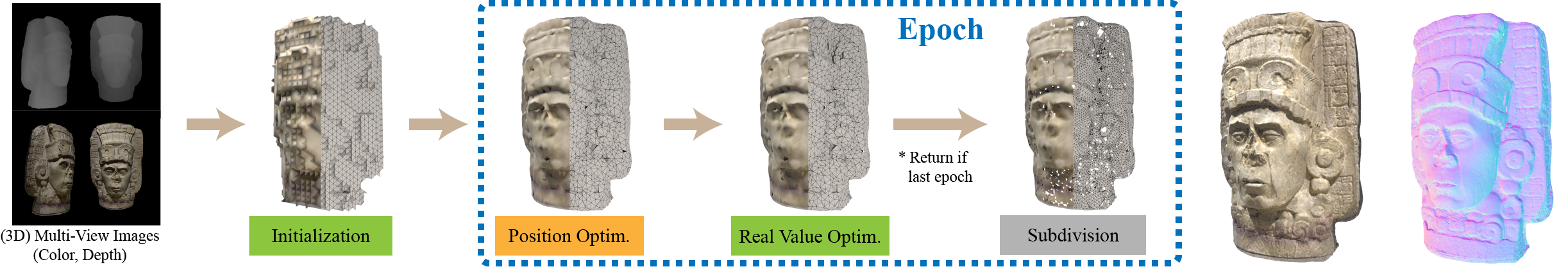}
    \caption{\textbf{Reconstruction process for 3D multi-view colored images of a sculpture.} In each stage, we optimize different per-point features: the \textcolor{Dandelion}{position} and the \textcolor{LimeGreen}{real ($\psi$)}, while the per-point color is refined at every stage. (Left) We display the meshes at each stage during the first epoch. (Right) We provide a rendering of the final mesh along with its view-point normal, which reveals the fine mesh details.}
    \label{fig:pipeline}
    \vspace{-1em}
\end{figure*}

The mesh generated by DMesh's tessellation function in~\cref{eq:tessellation-dmesh} is free from self-intersections because $\mathbb{F}_{wdt}$ itself is free from them. Additionally, it minimizes the occurrence of thin triangles—undesirable in many downstream tasks such as physics simulations—thanks to the properties of the WDT. However, computing the tessellation function is computationally expensive, as it requires calculating the WDT to define $\mathbb{F}_{wdt}$. In designing our \textit{Minimum-Ball}, we aimed to eliminate this computational bottleneck while preserving the favorable properties related to self-intersection avoidance and triangle quality. In the following, we demonstrate that \textit{Minimum-Ball} satisfies these requirements.

For a given set of points $\mathbb{P} \in \mathbb{R}^{d}$ and a face $F = \{p_1, p_2, ..., p_{d}\} \subset \mathbb{P}$, we define a bounding ball of $F$ as a $d$-dimensional ball that goes through every point of $F$. Note that this bounding ball is not unique, but there exists a unique minimum bounding ball, which has the minimum radius among every bounding ball. We name it as $B_{F}$ (\cref{fig:minball}). Then, we define $\mathbb{F}_{min}$ as a set of faces whose minimum bounding ball does not contain any other point in $\mathbb{P}$.
\begin{definition}
\label{def:min-ball}
    $F \in \mathbb{F}_{min}$ if and only if there is no point in $\mathbb{P}$ that lies (strictly) inside $B_{F}$. 
\end{definition}

Note that we can ignore points in $F$, as they are located on the boundary of $B_{F}$. In~\cref{fig:minball-point}, we render a 2D case, where $\overline{AB}$ does not satisfy this definition because of $F$. In contrast, $\overline{DF}$ satisfies this condition. Then, we can observe that $\mathbb{F}_{min}$ is a subset of faces in Delaunay Triangulation (DT) of $\mathbb{P}$ ($\mathbb{F}_{dt}$).
\begin{lemma}
\label{lemma:min-ball-dt}
    $F \in \mathbb{F}_{min} \Rightarrow F \in \mathbb{F}_{dt}$.
\end{lemma}
\begin{proof}
    By definition, a face $F$ is in $\mathbb{F}_{dt}$ if there is a bounding ball of $F$ that does not contain any other point in $\mathbb{P}$~\citep{cheng2013delaunay}. If the face $F$ is in $\mathbb{F}_{min}$, its minimum bounding ball satisfies this condition. Thus $F$ is in $\mathbb{F}_{dt}$.
\end{proof}

Note that $\mathbb{F}_{dt}$ is also free from self-intersections as $\mathbb{F}_{wdt}$, and thus is $\mathbb{F}_{min}$. Furthermore, it inherently minimizes the number of thin triangles, as guaranteed by DT. However, note that $\mathbb{F}_{min}$ does not necessarily tessellate the entire convex shape, as there could be faces in $\mathbb{F}_{dt}$ that are not in $\mathbb{F}_{min}$ (\eg $\overline{AB}$ in \cref{fig:minball-point}). 

Now, based on Definition~\ref{def:min-ball}, we can check if $F$ is in $\mathbb{F}_{min}$. Let us denote the center and radius of $B_{F}$ as $B^{c}_{F} \in \mathbb{R}^{d}$ and $B^{r}_{F}$. We can compute these values in a differentiable way (Appendix~\ref{appendix:min-ball-computation}). Then, we can compute the signed distance between $B_{F}$ and $\mathbb{P}$ as follows:
\begin{align}
    d(B_F, \mathbb{P}) &= \min_{p \in \mathbb{P} - F} ||p - B^{c}_{F}|| - B^{r}_{F}.
\end{align}

As shown above, we can easily find $d(B_{F}, \mathbb{P})$ by finding the nearest point of $B^{c}_{F}$ in $\mathbb{P} - F$. Using this signed distance, we can check if $F$ is in $\mathbb{F}_{min}$ as follows.
\begin{align}
    F \in \mathbb{F}_{min} \Leftrightarrow d(B_{F}, \mathbb{P}) > 0.
\end{align}

Then, we define $\Lambda_{min}$ with sigmoid function as
\begin{equation}
\label{eq:min-ball}
    \Lambda_{min}(F) = \sigma(d(B_{F}, \mathbb{P}) \cdot \alpha_{min}),
\end{equation}
where $\alpha_{min}$ is a constant (Appendix~\ref{appendix:min-ball-sigmoid-coef}).

With this formulation, we can evaluate~\cref{eq:tessellation_ours} far more efficiently than~\cref{eq:tessellation-dmesh}. Our method relies on a highly parallelizable nearest neighbor search algorithm\footnote{We used implementation of PyTorch3D~\citep{ravi2020accelerating}.}, unlike the sequential WDT. While WDT has a practical time complexity of $O(|\mathbb{P}|)$, it is relatively slow. In contrast, our approach has a time complexity of $O(|F| \cdot \log|\mathbb{P}|)$, where $|F|$ is the number of query faces to evaluate. However, by parallelizing the nearest neighbor search across query faces, especially on GPU, this complexity effectively reduces to $O(\log|\mathbb{P}|)$\footnote{We assume that $|F|$ does not increase exponentially, which is a practical assumption as query faces are often determined by local proximity.}. This allows our tessellation function to run up to 32 times faster in 3D than DMesh~\citep{son2024dmesh} (\cref{fig:minball-speed}). For optimization tasks like reconstruction, we further accelerate by periodically caching nearest neighbors for each query face (Appendix~\ref{appendix:nn-caching}). We provide formal algorithm in Appendix~\ref{appendix:min-ball-algo}.

%% file: sec/4_recon.tex
\section{Reconstruction Process}
\label{sec:recon-process}

The goal of reconstruction is to optimize point-wise features so that the resulting mesh aligns with the input observation. As shown in~\cref{fig:tsne}, the discrete mesh topology dynamically changes during the optimization of continuous features to better fit the given input.

\cref{fig:pipeline} provides an overview of our reconstruction process for recovering a 3D colored mesh from multi-view images. In the first stage, we initialize per-point features. If a point cloud is available, we use it to obtain a better starting point (see the initial mesh in~\cref{fig:tsne}); otherwise, we initialize with a regular tetrahedral grid. Next, we optimize the point positions while keeping the real values fixed, followed by optimizing the real values while fixing the point positions. In these two stages, we minimize loss functions tailored to each input modality (e.g., Chamfer Distance loss for point clouds, rendering loss for images). To increase mesh complexity and capture finer details, we subdivide the mesh by inserting additional points. This process is iterated for a fixed number of epochs. Finally, to accelerate the overall process and enhance the final mesh quality, we introduce several innovations at each stage. Detailed explanations are provided in Appendix~\ref{appendix:recon-detail}, due to lack of space.

%% file: sec/5_experiments.tex
\begin{figure}[t]
    \centering
    \includegraphics[width=1.0\linewidth]{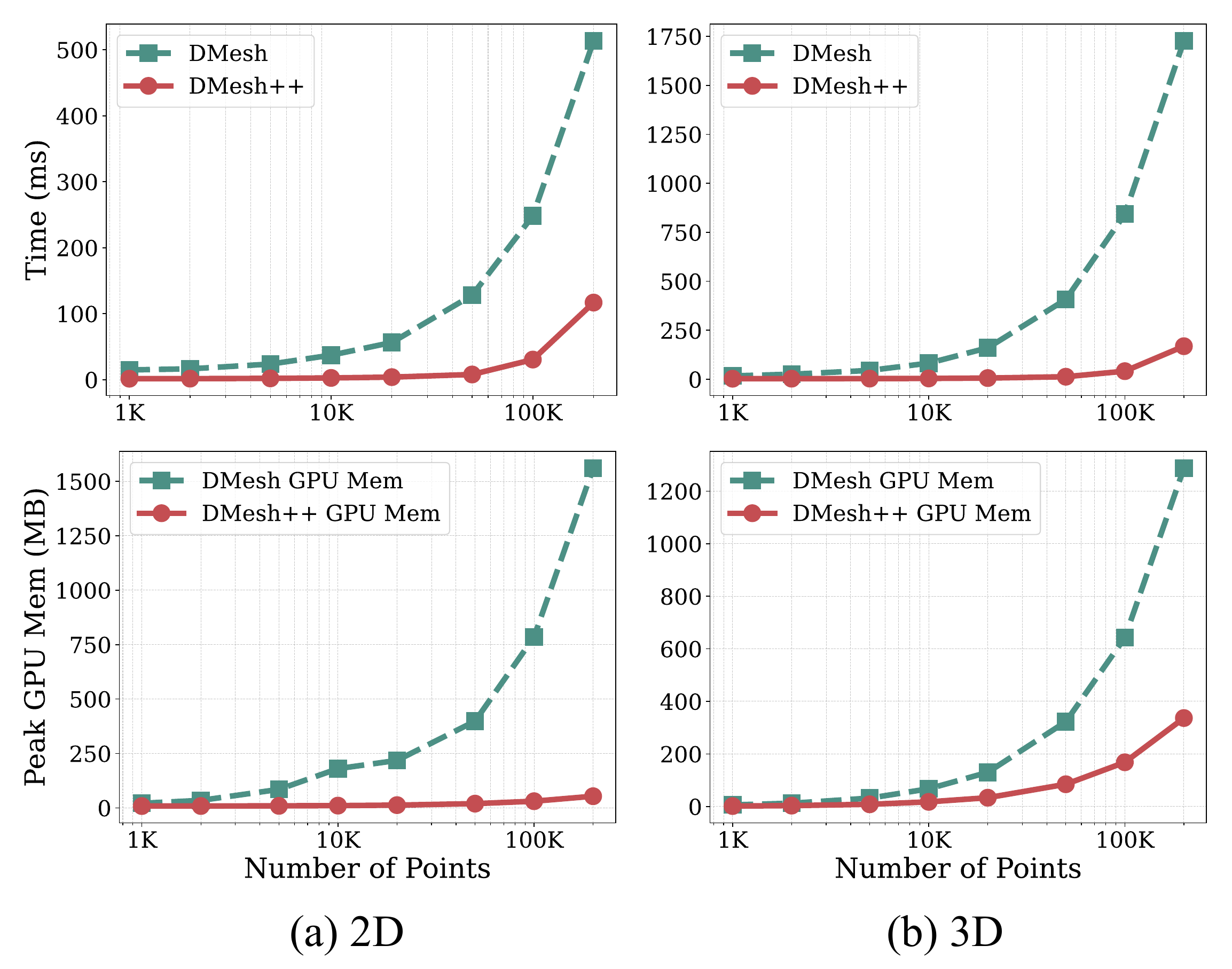}
    \vspace{-2em}
    \caption{\textbf{Comparison of tessellation cost.} Our method computes face probabilities up to 16 times faster in 2D and 32 times faster in 3D than DMesh~\citep{son2024dmesh}, while using up to 96\% less GPU memory in 2D and 75\% less in 3D.}
    \label{fig:minball-speed}
    \vspace{-1.5em}
\end{figure}

\section{Experiments}
\label{sec:experiments}

This section presents our experimental results. First, we evaluate how \textit{Minimum-Ball} enhances the computational efficiency of the tessellation function in both 2D and 3D. Next, we demonstrate the results of a point cloud reconstruction task in 2D and 3D as a practical application. Finally, we showcase an application in 3D multi-view reconstruction. Together, these results illustrate that our method is well-suited for downstream tasks involving complex shapes. Our main algorithms are implemented in PyTorch~\citep{paszke2017automatic} and CUDA~\citep{nickolls2008scalable}. All experiments were conducted on a system with an AMD EPYC 7R32 CPU and an NVIDIA A10 GPU. For details, refer to Appendix~\ref{appendix:exp-details}.

\subsection{Tessellation Cost}
\label{sec:exp-tess-speed}

We compare the computational cost of the tessellation function for DMesh~\citep{son2024dmesh} (\cref{eq:tessellation-dmesh}) and our DMesh++ (\cref{eq:tessellation_ours}). For both 2D and 3D scenarios, we randomly generate $N$ points within a unit cube, find each point's 10 nearest neighbors, and use these proximities to form potential face combinations. From these, we randomly select $N$ faces as query faces for the tessellation function. For each value of $N$ (ranging from $1K$ to $200K$ to reflect practical scenarios), we conducted 5 trials and averaged the computational costs.

In~\cref{fig:minball-speed}, we compare the computational costs of DMesh and DMesh++. In terms of speed, DMesh's performance scales linearly with the number of points in both 2D and 3D due to its sequential WDT algorithm. In contrast, DMesh++ exhibits sub-linear scaling up to 50K points, benefiting from GPU parallelization (\cref{sec:min-ball}). Beyond this range, computational costs increase more sharply because of GPU thread limitations; however, DMesh++ still processed 200K points in $117ms$ for 2D and $168ms$ for 3D. Regarding GPU memory usage, both methods scale linearly, but DMesh++ uses significantly less memory since it stores only the additional information related to the minimum balls, whereas DMesh must store all details of the power diagram on the GPU\footnote{Note that the 2D version of DMesh requires more memory than the 3D version, as the 2D implementation is not as optimized using CUDA.}.

These results demonstrate that the \textit{Minimum-Ball} algorithm significantly enhances tessellation efficiency, enabling the effective handling of complex shapes.

\subsection{Reconstruction Tasks}
\label{sec:exp-recon-task}

\begin{figure}[t]
    \centering
    \includegraphics[width=0.99\linewidth]{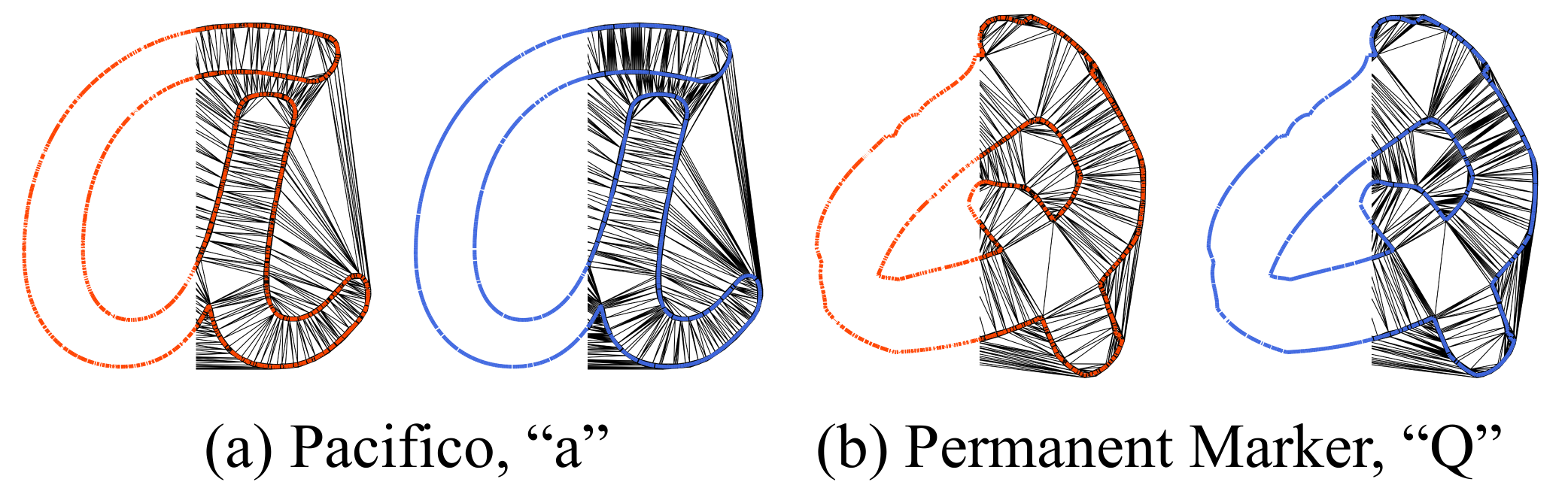}
    \vspace{-0.5em}
    \caption{\textbf{Qualitative comparison of 2D point cloud reconstruction results.} The outputs of DMesh~\citep{son2024dmesh} and DMesh++ are rendered in \red{red} and \blue{blue}, respectively.}
    \label{fig:font-qualitative}
    \vspace{-0.5em}
\end{figure}

\begin{table}[t]
    \centering
    \scalebox{0.7}{
  \begin{tabular}{l|cccc}
    \toprule
    Method & CD{\footnotesize ($\times 10^{-6}$)}$\downarrow$ & \# Verts. & \# Edges. & Time (sec)\\
    \midrule
DMesh~\cite{son2024dmesh} & 1.97 & 2506 & 2245 & 30.39 \\
    \midrule
DMesh++ & 1.82 & 2862 & 2793 & 11.33 \\
    \bottomrule
  \end{tabular}
  }
  \caption{\textbf{Quantitative comparison of the 2D point cloud reconstruction task for the font dataset.} DMesh++ reconstructs 2D meshes with greater accuracy and efficiency than DMesh.} 
\label{tab:font-quantitative}
\vspace{-1.5em}
\end{table}

\begin{table*}[t!]
    \centering
    \scalebox{0.8}{
  \begin{tabular}{l|ccccc|cccc|ccc}
    \toprule
     & \multicolumn{5}{c|}{Geometric Accuacy}  & \multicolumn{4}{c|}{Mesh Quality} & \multicolumn{3}{c}{Statistics}\\
    \midrule
    Method & CD{\footnotesize ($\times 10^{-3}$)}$\downarrow$ &  F1$\uparrow$ &  NC$\uparrow$ &  ECD$\downarrow$ &  EF1$\uparrow$ & AR$\downarrow$ & SI$\downarrow$ & NME$\downarrow$ & NMV$\downarrow$ & \# Verts. & \# Faces. & Time (sec)\\
    \midrule
    \midrule
VoroMesh~\cite{maruani2023voromesh}
  & 19.591 
  & 0.352
  & 0.855
  & \colorbox{GreenYellow}{0.054}
  & 0.072
  & 145.7
  & 0
  & \colorbox{GreenYellow}{0.001} 
  & \colorbox{GreenYellow}{0} 
  & 64561 
  & 129338
  & 11 \\
\midrule
DMesh++ 
  & \colorbox{GreenYellow}{0.034}
  & \colorbox{GreenYellow}{0.471}
  & \colorbox{GreenYellow}{0.919}
  & 0.063
  & \colorbox{GreenYellow}{0.094}
  & \colorbox{GreenYellow}{1.765}
  & 0
  & 0.130 
  & 0.003
  & 25415
  & 58537
  & 282 \\ 
  \midrule
  \midrule
PSR~\cite{kazhdan2013screened} 
  & 10.164
  & 0.392
  & \colorbox{GreenYellow}{0.943}
  & 0.302
  & 0.026
  & 5.218
  & 0
  & \colorbox{GreenYellow}{0} 
  & \colorbox{GreenYellow}{0}  
  & 139857 
  & 279739
  & 4 \\
PoNQ~\cite{maruani2024ponq}
  & 1.578
  & 0.402
  & 0.934
  & \colorbox{GreenYellow}{0.056}
  & 0.090
  & 2.288
  & 0
  & 0.002 
  & \colorbox{GreenYellow}{0} 
  & 47254 
  & 94664
  & 32 \\
DMesh~\cite{son2024dmesh} 
  & 0.154 
  & 0.289
  & 0.921
  & 0.077
  & 0.069
  & 1.961
  & 0
  & 0.103 
  & {0.002}
  & 5815
  & 13088
  & 1147 \\
\midrule
DMesh++
  & \colorbox{GreenYellow}{0.033} 
  & \colorbox{GreenYellow}{0.480} 
  & {0.938}
  & 0.060
  & \colorbox{GreenYellow}{0.116}
  & \colorbox{GreenYellow}{1.814}
  & 0
  & 0.087 
  & 0.004
  & 25396
  & 55546
  & 282
  \\
\bottomrule
  \end{tabular}
  }
    \caption{\textbf{Quantitative comparison of 3D point cloud reconstruction results over 50 manually chosen models from Thingi10K~\citep{zhou2016thingi10k} and Objaverse~\citep{deitke2023objaverse}.} We highlight the \colorbox{GreenYellow}{best} results for each metric. The upper panel compares methods that use \textit{unoriented} point clouds (VoroMesh, DMesh++), while the lower panel displays methods that use \textit{oriented} point clouds (PSR, PoNQ, DMesh, DMesh++).}
    \label{tab:3d-pc-quant}
    \vspace{-1.5em}
\end{table*}

\paragraph{Dataset.} 

For the 2D point cloud reconstruction task, we used vector graphics of 26 letters from four different font styles downloaded from the Google Fonts service\footnote{\url{https://fonts.google.com/}}. Additionally, we employed six vector graphic images representing complex drawings from Adobe Stock\footnote{\url{https://stock.adobe.com/}}.

For the 3D reconstruction tasks, we randomly selected 500 models from Thingi10K~\citep{zhou2016thingi10k} dataset, half of which are closed-surface models and the other half are open-surface models. Also, we manually chose 20 and 30 models from Thingi10K and Objaverse~\citep{deitke2023objaverse} dataset for better diversity. We normalized the model scale to $1$ before reconstruction.

\paragraph{Metrics.} 

For 2D results, we report the Chamfer Distance (CD) of the reconstructed outputs for quantitative comparison. For 3D, we assess geometric accuracy using five metrics—CD (Chamfer Distance), F1 (F1 score), NC (Normal Consistency), ECD (Edge Chamfer Distance), and EF1 (Edge F1 score)—following~\citep{chen2022neural}. Additionally, we evaluate mesh quality using four metrics: AR (Aspect Ratio), SI (Self-Intersection ratio), NME (Non-Manifold Edge ratio), and NMV (Non-Manifold Vertex ratio). Finally, we provide statistics on mesh complexity (e.g., vertex and face counts) along with the reconstruction time.

\subsubsection{2D Point Cloud Reconstruction}
\label{sec:exp-2d-pc}

\begin{figure}[t]
    \centering
    \includegraphics[width=\linewidth]{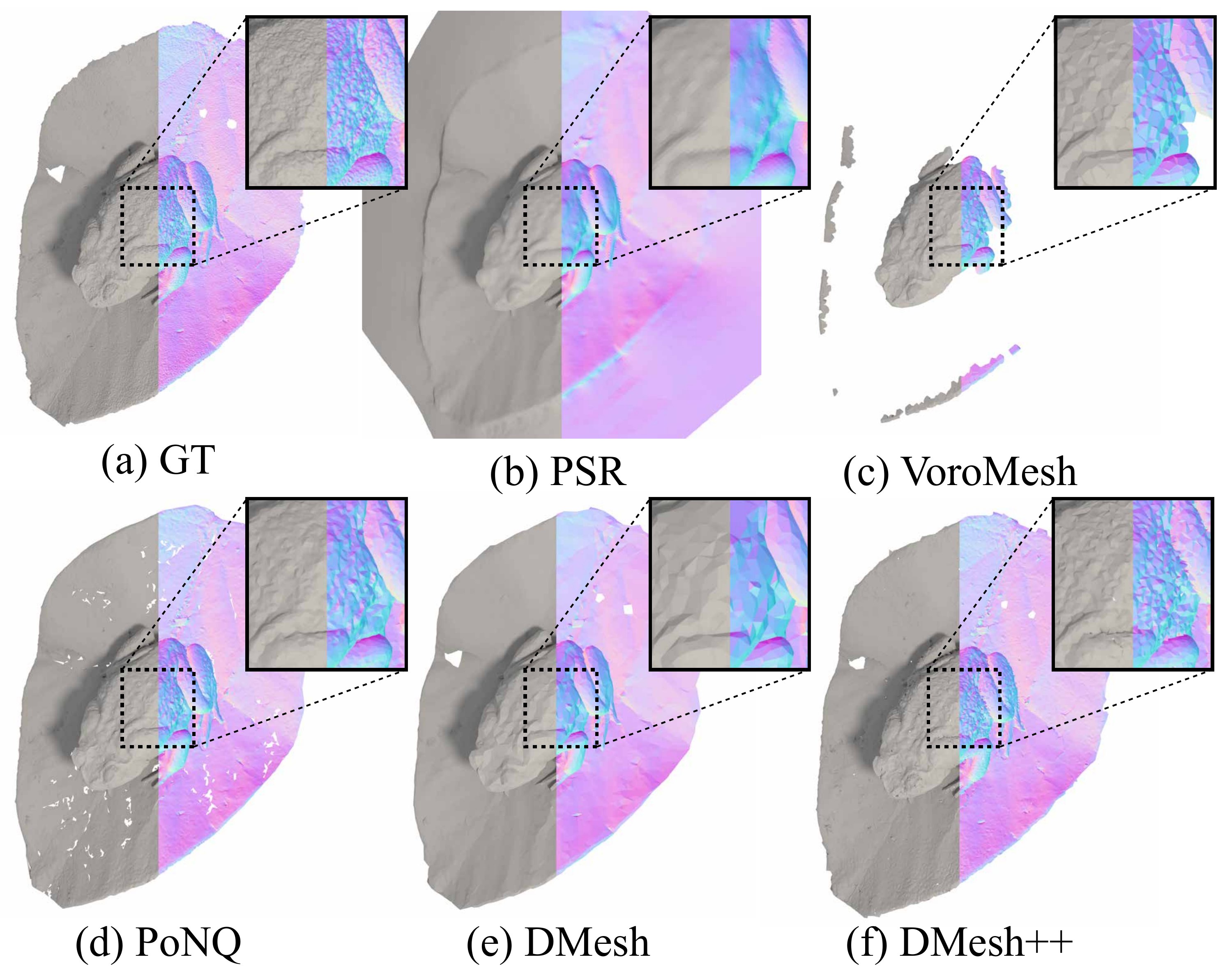}
    \caption{\textbf{Qualitative comparison of 3D point cloud reconstruction results for a toad sitting on a leaf.} For each result, we render its diffuse image on the left, and view-point normals on the right.}
    \label{fig:3d-pc-qual}
    \vspace{-1.5em}
\end{figure}

In this task, we aim to reconstruct a 2D mesh from a 2D point cloud to demonstrate that the DMesh++ formulation is applicable in 2D. We primarily compare our reconstruction results on a font dataset with those obtained by DMesh~\citep{son2024dmesh}, whose formulation is also easily extendable to 2D. As input, we sampled 1,000 points from each spline curve composing the font and downsampled the entire point cloud using a grid with a cell size of 0.005. 

In~\cref{tab:font-quantitative}, we present a quantitative comparison of the reconstruction results. We observe that DMesh++ reconstructs 2D meshes more faithfully than DMesh in terms of CD loss, while also running $2.6\times$ faster. In~\cref{fig:font-qualitative}, we render the reconstructed 2D meshes for two examples. The results show that DMesh++ produces significantly fewer holes compared to DMesh, which is consistent with the CD loss comparison. However, some of these holes are inevitable, as we mainly reconstruct our shape with CD loss and there are places that lack points.

Additionally, we reconstructed several complex 2D drawings from point clouds to demonstrate the computational efficiency of DMesh++. Since these drawings contain finer details than the fonts, we downsampled the entire point cloud using a grid with a cell size of $0.001$. When we attempted to reconstruct these drawings with DMesh, the GPU memory consumption became prohibitively high, resulting in an error during reconstruction. Therefore, for qualitative evaluation, we report only the results of DMesh++ in~\cref{fig:teaser} and Appendix~\ref{appendix:2d-drawing-recon}.

Before moving on, we'd like to introduce an experimental algorithm that is applicable to 2D mesh optimization called the \textit{Reinforce-Ball} algorithm, which could be used for producing an efficient mesh that adapts to local geometry. Please see Appendix~\ref{appendix:rl-ball} for more details.

\begin{table}[t]
    \centering
    \scalebox{0.68}{
  \begin{tabular}{l|ccccc}
    \toprule
    Method & CD{\footnotesize ($\times 10^{-4}$)}$\downarrow$ & F1$\uparrow$ & NC$\uparrow$ & ECD$\downarrow$ & EF1$\uparrow$\\
    \midrule
PSR~\cite{kazhdan2013screened} & 12.1 / 16.3 & 0.47 / 0.45 & \red{0.97} / \red{0.96} & 0.45 / 0.37 & 0.01 / 0.01 \\
PoNQ~\cite{maruani2024ponq} & \blue{2.44} / 7.86 & \blue{0.48} / \blue{0.45} & \blue{0.97} / 0.95 & \red{0.01} / 0.04 & \red{0.26} / \blue{0.22} \\
DMesh~\cite{son2024dmesh} & 2.82 / \blue{3.29} & 0.25 / 0.22 & 0.94 / 0.93 & \blue{0.01} / \red{0.01} & 0.15 / 0.12 \\
    \midrule
DMesh++ & \red{0.37} / \red{0.37} & \red{0.48} / \red{0.47} & 0.96 / \blue{0.95} & 0.02 / \blue{0.02} & \blue{0.25} / \red{0.23} \\
    \bottomrule
  \end{tabular}
  }
  \caption{\textbf{Quantitative comparison of 3D point cloud reconstruction results over 500 randomly chosen models from Thingi10K~\citep{zhou2016thingi10k}.} The results over closed surfaces and open surfaces are shown together: (closed / open). We highlight the best results for two different categories with \red{red} and \blue{blue}, respectively.} 
\label{tab:pc-quantitative-thingi}
\vspace{-1.5em}
\end{table}

\subsubsection{3D Point Cloud Reconstruction}

\begin{figure*}
    \centering
    \includegraphics[width=\linewidth]{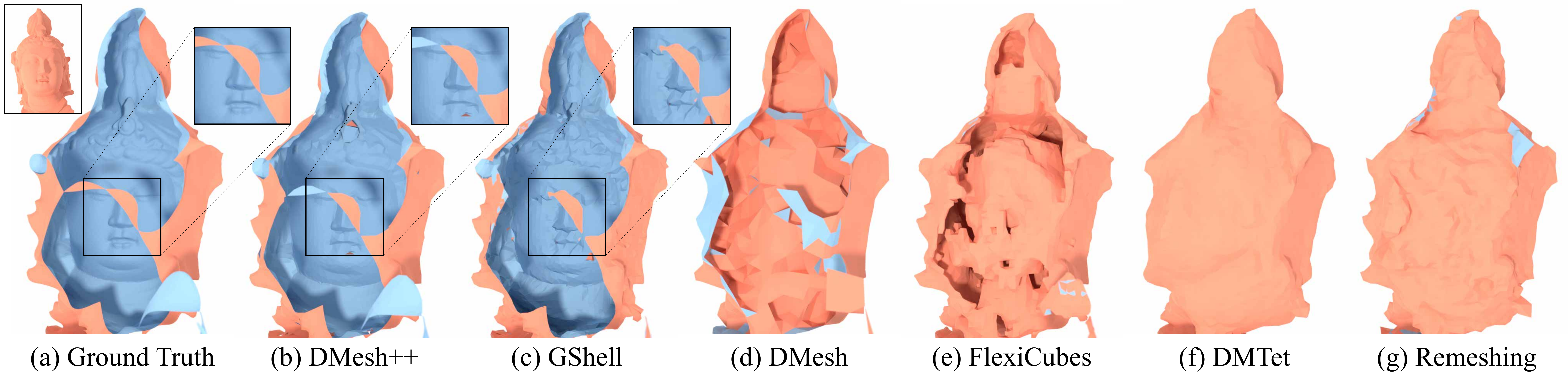}
    \caption{\textbf{Qualitative comparison of 3D multi-view reconstruction results for open surface.} Here we illustrate from back of an open surface model (the front view is rendered at the left top of (a)). Colors represent {\color{insideblue} inside} and {\color{outsidepink} outside} facing surfaces. 
    }
    \label{fig:mv-qual}
    \vspace{-1em}
\end{figure*}

\begin{table*}[t!]
    \centering
    \scalebox{0.80}{
  \begin{tabular}{l|ccccc|cccc|ccc}
    \toprule
     & \multicolumn{5}{c|}{Geometric Accuacy}  & \multicolumn{4}{c|}{Mesh Quality} & \multicolumn{3}{c}{Statistics}\\
    \midrule
    Method & CD{\footnotesize ($\times 10^{-3}$)}$\downarrow$ &  F1$\uparrow$ &  NC$\uparrow$ &  ECD$\downarrow$ &  EF1$\uparrow$ & AR$\downarrow$ & SI$\downarrow$ & NME$\downarrow$ & NMV$\downarrow$ & \# Verts. & \# Faces. & Time (sec)\\
    \midrule
Remeshing~\cite{palfinger2022continuous} & 0.977 & \colorbox{GreenYellow}{0.343} & \colorbox{GreenYellow}{0.907} & \colorbox{GreenYellow}{0.088} & 0.044 & \colorbox{Lavender}{1.562} & 0.220 & \colorbox{Lavender}{0} & \colorbox{Lavender}{0} & 13273 & 26540 & 25 \\
DMTet~\cite{shen2021deep} & 1.395 & 0.191 & 0.868 & 0.145 & 0.032 & 6.175 &\colorbox{Lavender}{0} & \colorbox{Lavender}{0} & \colorbox{Lavender}{0} & 20549 & 41131 & 201 \\
FlexiCubes~\cite{shen2023flexible} & 3.493 & 0.290 & 0.880 & 0.091 & 0.038 & 2.093 & \colorbox{GreenYellow}{0.011} & \colorbox{Lavender}{0} & 0.005 & 14811 & 28882 & 88 \\
GShell~\cite{liu2023ghost} & 5.807 & 0.326 & 0.904 & 0.116 & \colorbox{GreenYellow}{0.047} & 2.793 & 0.047 & \colorbox{Lavender}{0} & 0.004 & 14587 & 28054 & 209 \\
DMesh~\cite{son2024dmesh} & \colorbox{GreenYellow}{0.697} & 0.328 & 0.898 & 0.104 & 0.045 & 1.820 & \colorbox{Lavender}{0} & 0.042 & \colorbox{GreenYellow}{0.002} & 2461 & 5058 & 772 \\
\midrule
DMesh++ & \colorbox{Lavender}{0.342} & \colorbox{Lavender}{0.360} & \colorbox{Lavender}{0.923} & \colorbox{Lavender}{0.074} & \colorbox{Lavender}{0.059} & \colorbox{GreenYellow}{1.639} & \colorbox{Lavender}{0} & \colorbox{GreenYellow}{0.025} & 0.036 & 12507 & 23727 & 205 \\
\bottomrule
\end{tabular}
  }
    \caption{\textbf{Quantitative comparison of 3D multi-view reconstruction results over 50 manually chosen models from Thingi10K~\citep{zhou2016thingi10k} and Objaverse~\citep{deitke2023objaverse}.} We highlight the \colorbox{Lavender}{best} results and the \colorbox{GreenYellow}{second best} results for each evaluation metric.}
    \label{tab:mv-quant}
    \vspace{-1.5em}
\end{table*}

In this task, we reconstruct a 3D mesh from a dense 3D point cloud. As input, we sampled 200K points from the ground truth mesh using the Poisson disk sampling algorithm~\citep{bridson2007fast} implemented in MeshLab~\citep{cignoni2008meshlab}. For comparison, we employed other widely used optimization-based point cloud reconstruction methods, including Screened Poisson Surface Reconstruction (PSR)\citep{kazhdan2013screened}, VoroMesh\citep{maruani2023voromesh}, PoNQ~\citep{maruani2024ponq}, and DMesh~\citep{son2024dmesh}. For PSR, DMesh, and PoNQ, we used oriented point clouds for reconstruction, while for DMesh++ we tested both unoriented and oriented point clouds. We used the PSR implementation available in MeshLab. For DMesh, we used its default settings; for VoroMesh and PoNQ, we employed a grid size of 128 and optimized for 1000 epochs to achieve the best results.

We report the quantitative comparisons in~\cref{tab:3d-pc-quant,tab:pc-quantitative-thingi}. In~\cref{tab:3d-pc-quant}, results are averaged over the 50 manually chosen models and split into two categories: methods using \textit{unoriented} versus \textit{oriented} point clouds. In both categories, DMesh++ performs best or is comparable to other methods in geometric accuracy and triangle aspect ratio. Compared to DMesh, DMesh++ outperforms all metrics while handling 4.2$\times$ more faces with a 76\% reduction in computation time. Furthermore, DMesh++ achieves significantly better CD, F1, and EF1 scores than PSR, VoroMesh, and PoNQ, which struggle with open surfaces. Separate evaluations in~\cref{tab:pc-quantitative-thingi} over 500 randomly chosen models confirm that DMesh++ largely outperforms on open surfaces and performs comparably or better on closed surfaces.

These quantitative results align with the qualitative comparison in~\cref{fig:3d-pc-qual}, where a toad sitting on a thin leaf is reconstructed . While PSR and VoroMesh fail to reconstruct the open surface of the leaf and PoNQ produces many holes, DMesh captures the overall geometry but misses fine details. DMesh++ successfully recovers both the overall shape and the intricate details. For qualitative comparison on a closed surface, see~\cref{fig:3d-pc-qual-closed}.


\subsubsection{3D Multi-View Reconstruction}
\label{sec:exp-3d-recon}

In this task, we reconstruct a mesh from multi-view images of a target object, assuming full knowledge of the rendering model and lighting conditions. Specifically, we employ simple Phong shading~\citep{phong1998illumination} with a directional light from the camera to the object for rendering the ground truth images. We generate $(512\times512)$ diffuse and depth maps of the object from 64 viewpoints (\cref{fig:3d-mv-qual-closed}) to supervise the reconstruction process. Note that we omit colors and textures in this experiment to focus solely on geometric quality.

\begin{table}[t]
    \centering
    \scalebox{0.70}{
  \begin{tabular}{l|ccccc}
    \toprule
    Method & CD{\footnotesize ($\times 10^{-4}$)}$\downarrow$ & F1$\uparrow$ & NC$\uparrow$ & ECD$\downarrow$ & EF1$\uparrow$\\
    \midrule
REM~\citep{palfinger2022continuous} & 58.0 / 27.5 & 0.33 / 0.32 & 0.90 / 0.90 & 0.06 / 0.07 & 0.11 / 0.09 \\
DMT~\cite{shen2021deep} & 36.2 / 65.1 & 0.21 / 0.21 & 0.89 / 0.88 & 0.17 / 0.13 & 0.03 / 0.03 \\
FLE~\cite{shen2023flexible} & 14.5 / 34.1 & 0.37 / 0.35 & 0.92 / 0.91 & 0.04 / 0.06 & 0.10 / 0.08 \\
GSH~\cite{liu2023ghost} & \blue{5.74} / 54.6 & \red{0.38} / \blue{0.36} & \red{0.94} / \blue{0.93} & \blue{0.04} / \blue{0.06} & \blue{0.15} / \blue{0.12}  \\
DME~\cite{son2024dmesh} & 11.9 / \blue{11.7} & 0.18 / 0.20 & 0.87 / 0.87 & 0.06 / 0.08 & 0.04 / 0.04 \\
\midrule
DMesh++ & \red{3.35} / \red{3.82} & \blue{0.37} / \red{0.36} & \blue{0.94} / \red{0.93} & \red{0.03} / \red{0.03} & \red{0.21} / \red{0.17} \\
    \bottomrule
  \end{tabular}
  }
  \caption{\textbf{Quantitative comparison of 3D multi-view reconstruction results over 500 randomly chosen models from Thingi10K~\citep{zhou2016thingi10k}.} The results over closed surfaces and open surfaces are shown together: (closed / open). We highlight the best results for two different categories with \red{red} and \blue{blue}, respectively.} 
\label{tab:mv-quantitative-thingi}
\vspace{-1.5em}
\end{table}

For comparison, we evaluated five mesh reconstruction algorithms: Remeshing~\citep{brandt1992continuous}, DMTet~\citep{shen2021deep}, FlexiCubes~\citep{shen2023flexible}, GShell~\citep{liu2023ghost}, and DMesh~\citep{son2024dmesh}. We optimized each method to produce the best-quality meshes with similar vertex and face counts. Specifically, we set the grid sizes to 128 for DMTet, and 80 for both FlexiCubes and GShell. For more details, please refer to Appendix~\ref{appendix:3d-exp-details}.

In~\cref{tab:mv-quant,tab:mv-quantitative-thingi} and~\cref{fig:mv-qual,fig:self-intersection}, we present quantitative and qualitative comparisons of the reconstruction results. In~\cref{tab:mv-quant}, we observe that Remeshing, DMTet, and FlexiCubes have high CD errors, largely because they cannot represent open surfaces (\cref{fig:mv-qual}). This limitation also explains why Remeshing and FlexiCubes are faster than other methods. Although Remeshing achieved the best AR and avoided non-manifoldness, it generated numerous self-intersections, particularly on open surfaces (\cref{fig:self-intersection}).

\begin{figure}
    \centering
    \includegraphics[width=0.9\linewidth]{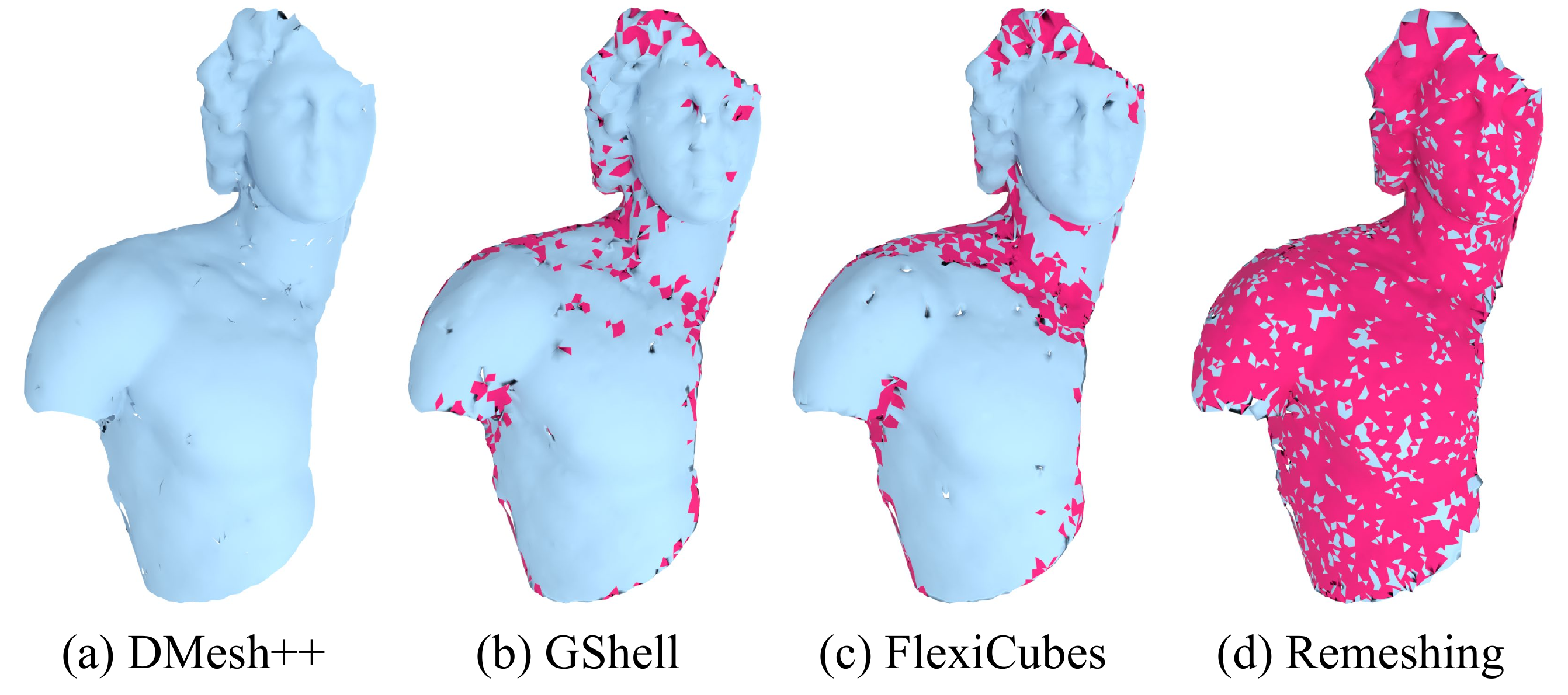}
    \caption{\textbf{Self-intersection of the reconstructed mesh.} The self-intersected faces of the mesh are rendered in red.}
    \label{fig:self-intersection}
    \vspace{-1.5em}
\end{figure}

When it comes to GShell, we found out that it usually improves CD loss by representing open surfaces through sub-surface extraction from closed templates. However, it still struggles with complex open surfaces (\cref{fig:mv-qual}). Also, it had a robustness issue with several outlier models -- which is the main reason of high average CD, even though it achieved much better CD for most of the models. Additionally, it also suffers from self-intersections (\cref{fig:self-intersection}) and suboptimal AR. Compared to that, DMesh produced self-intersection free mesh with better CD and AR, as it can represent open surfaces robustly. Also, it produced much simpler mesh than the other methods. However, it sometimes produced false inner structure due to occlusion (\cref{fig:mv-qual}), and its largest drawback was in computational cost, limiting its utility for fine-grained reconstructions. 

DMesh++ addresses this issue using the \emph{Minimum-Ball} algorithm (\cref{sec:min-ball}) and nearest neighbor caching (Appendix~\ref{appendix:nn-caching}). It achieves the best or comparable results across all metrics while maintaining computational costs similar to those of other methods. Furthermore, as shown in~\cref{tab:mv-quantitative-thingi}, DMesh++ delivers superior, or at least comparable performance for both closed and open surfaces. Qualitative comparisons in~\cref{fig:mv-qual,fig:3d-mv-qual-closed} also support this observation. These results prove the robustness of DMesh++ in handling complex shapes with diverse topologies, whereas other methods are limited in one aspect or another.

\vspace{0.5em}
\noindent\textbf{Colors.} As mentioned in~\cref{fig:mesh_dmesh} and~\cref{sec:preliminary}, points can carry additional features. Here, we demonstrate that we can jointly optimize per-point colors to recover a textured shape from multi-view images. For a point on a face, the color of the point is determined by barycentric interpolation of the face vertex colors. Under the assumption that we know all the rendering and lighting conditions, we can reconstruct textured meshes from multi-view images as shown in~\cref{fig:teaser,fig:pipeline,fig:physics-sim}. In particular,~\cref{fig:physics-sim} demonstrates that DMesh++ can reconstruct a small scene on which physics simulations, such as bouncing balls, can be run directly. These results highlight a promising future direction of jointly optimizing additional per-point features. 


%% file: sec/6_conclusion.tex
\section{Conclusion}
\label{sec:discussion}


We presented DMesh++, a probabilistic approach for efficient, differentiable mesh connectivity handling. Our \emph{Minimum-Ball} algorithm significantly reduces computational costs, enabling DMesh++ to recover 2D and 3D shapes with diverse topologies from point clouds or multi-view images more effectively than baseline methods. 

\vspace{4pt}\noindent\textbf{Limitations.}  There are areas where DMesh++ can be further improved.
First, some non-manifoldness errors remain in the reconstruction results (\cref{tab:3d-pc-quant,tab:mv-quant}), while observing significant improvement over DMesh~\citep{son2024dmesh}. We conjecture that there is a trade-off between the representation power and the occurrence of these topological errors; more careful analysis of the trade-off is needed to maintain expressiveness while eliminating errors.
Second, there are application-specific challenges. Although our method yields superior 3D point cloud reconstruction results, it incurs higher computational costs (\cref{tab:3d-pc-quant}), and it is less effective for sparse point clouds. In 3D multi-view reconstruction, we cannot use the current implementation for real-world images, as discussed in Appendix~\ref{appendix:mv-limit}.

\begin{figure}
    \centering
    \includegraphics[width=\linewidth]{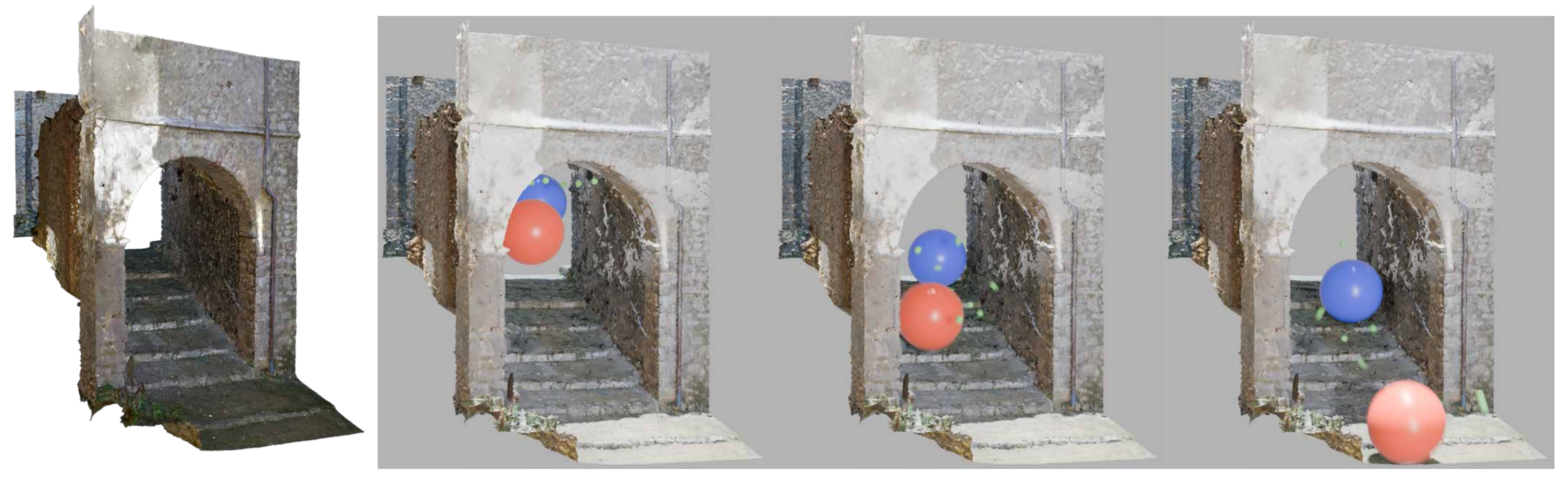}
    \caption{\textbf{Physics simulation on a staircase reconstructed from multi-view images.} We simulate the motion of bouncing balls directly on the mesh generated by DMesh++.}
    \label{fig:physics-sim}
    \vspace{-1em}
\end{figure}

\vspace{4pt}\noindent\textbf{Future Directions.} For 3D point cloud reconstruction, we plan to explore alternative loss formulations to CD loss, as its computation currently dominates the overall cost.

To extend our 3D multi-view reconstruction algorithm to real-world images, we will integrate DMesh++ with alternative representations such as Gaussian Splatting (GS)~\citep{kerbl20233d, huang20242d}. By appending GS features to the per-point features and optimizing them jointly, we expect to recover high-quality meshes that we can readily use for downstream tasks. 

Furthermore, we envision leveraging DMesh++ to train generative models that capture complex mesh connectivity. Currently, the \emph{Minimum-Ball} algorithm only uses point positions for connectivity. In future work, we plan to incorporate additional per-point features into this algorithm to train generative models capable of understanding diverse and intricate mesh connectivity, such as that of DNA (\cref{fig:teaser}).

We believe our work establishes an important foundation for harnessing differentiable, probabilistic mesh within the current optimization framework, and we hope it will further drive future downstream applications.

%% file: sec/X_suppl.tex
\clearpage
\setcounter{page}{1}
\maketitlesupplementary





\section{Details about \emph{Minimum-Ball} algorithm}

\subsection{Algorithm}
\label{appendix:min-ball-algo}

\begin{algorithm}
\caption{\texttt{Minimum-Ball}}
\label{alg:min-ball}
\begin{algorithmic}[1]
    \State $\mathbb{P, F} \gets$ Set of points and query faces
    \State $\alpha_{min} \gets$ Coefficient for sigmoid function
    \State $B_{\mathbb{F}}^{c}, B_{\mathbb{F}}^{r} \gets$ \texttt{Compute-Minimum-Ball($\mathbb{P, F}$)}
    \State $P^{nearest}_{\mathbb{F}} \gets$ \texttt{Find-Nearest-Neighbor($B_{\mathbb{F}}^{c}$, $\mathbb{P}$)}
    \State $d(B_{\mathbb{F}}, \mathbb{P}) \gets B_{\mathbb{F}}^{r} - ||P^{nearest}_{\mathbb{F}} - B^{c}_{\mathbb{F}}||$
    \State $\lambda_{min}(\mathbb{F}) \gets \sigma(d(B_{\mathbb{F}}, \mathbb{P}) \cdot \alpha_{min})$
    \State \textbf{return} $\lambda_{min}(\mathbb{F})$
\end{algorithmic}
\end{algorithm}

We formally describe the \emph{Minimum-Ball algorithm} in~\cref{alg:min-ball}. 

\begin{itemize}
    \item \textbf{Line 1:} We define the given set of points (specifically, their positions) as $\mathbb{P}$ and the query faces as $\mathbb{F}$.
    \item \textbf{Line 2:} We introduce $\alpha_{min}$, the coefficient for the sigmoid function used to map the signed distance to a probability. Details on determining $\alpha_{min}$ are provided in Appendix~\ref{appendix:min-ball-sigmoid-coef}.
    \item \textbf{Line 3:} For each query face $F \in \mathbb{F}$, we compute the minimum bounding ball ($B_F$) as described in Appendix~\ref{appendix:min-ball-computation}. We denote the entire set of bounding balls as $B_{\mathbb{F}}$, their centers as $B_{\mathbb{F}}^{c}$, and their radii as $B_{\mathbb{F}}^{r}$.
    \item \textbf{Line 4:} For each $F \in \mathbb{F}$, we find the nearest neighbor of $B_F^c$ in $\mathbb{P} - F$. However, this operation cannot be parallelized across all query faces because the set $\mathbb{P} - F$ varies for each face. To address this, we find $(d + 1)$-nearest neighbors of $B_F^c$ in $\mathbb{P}$, where $d$ is the spatial dimension. This approach ensures correctness in two scenarios:
    \begin{itemize}
        \item If $F \in \mathbb{F}_{min}$, the bounding ball $B_F$ does not contain any points from $\mathbb{P}$ within it, and the points on $F$ are the $d$-nearest neighbors of $B_F^c$. To find the nearest neighbor in $\mathbb{P} - F$, we need to consider $(d + 1)$-nearest neighbors.
        \item If $F \notin \mathbb{F}_{min}$, only the single nearest neighbor of $B_F^c$ is relevant.
    \end{itemize}
    To safely handle both cases, we always search for $(d + 1)$-nearest neighbors and then select the first neighbor from the list that does not belong to $F$.
    \item \textbf{Line 5:} We compute the signed distance $d(B_{\mathbb{F}}, \mathbb{P})$ for all query faces.
    \item \textbf{Line 6:} The signed distance is converted to a probability using the sigmoid function, with $\alpha_{min}$ as the scaling factor.
    \item \textbf{Line 7:} Finally, the algorithm returns the computed probabilities for all faces.
\end{itemize}

\subsection{Minimum-Ball computation}
\label{appendix:min-ball-computation}

Let us define a face $F = \{p_{1}, p_{2}, \dots, p_{d} \}$, where $p_{i} \in \mathbb{P}$. To determine the bounding balls of $F$, we first identify the set of points that are equidistant from the vertices of $F$. Among these, we select the point lying on the hyperplane containing $F$ as the center of the minimum bounding ball, denoted as $B_{F}^{c}$. 

When $d = 2$, the center simplifies to the midpoint of $F$:
\begin{equation}
    B_{F}^{c} \big|_{d=2} = \frac{1}{2}(p_{1} + p_{2}).
\end{equation}

For $d = 3$, the computation is more complex~\footnote{Derived from the Geometry Junkyard: \url{https://ics.uci.edu/~eppstein/junkyard/circumcenter.html}}:
\begin{equation}
\scriptsize{
    B_{F}^{c} \big|_{d=3} = p_{1} + \frac{||d_{2}||^{2} (d_1 \times d_2) \times d_1 + ||d_{1}||^{2} (d_2 \times d_1) \times d_2}{2||d_{1} \times d_{2}||^{2}},
}
\end{equation}
where $d_1 = p_2 - p_1$ and $d_2 = p_3 - p_1$.

Unlike the case where $d = 2$, for $d = 3$, we cannot compute $B_{F}^{c}$ if $||d_{1} \times d_{2}|| = 0$. During computation, cases where this value falls below a certain threshold are marked and excluded from subsequent steps to avoid numerical instability.

After determining $B_{F}^{c}$, we calculate the radius $B_{F}^{r}$ as the distance between $B_{F}^{c}$ and the points on $F$.

\subsection{Sigmoid coefficient $\alpha_{min}$}
\label{appendix:min-ball-sigmoid-coef}

\begin{figure}
    \centering
    \includegraphics[width=\linewidth]{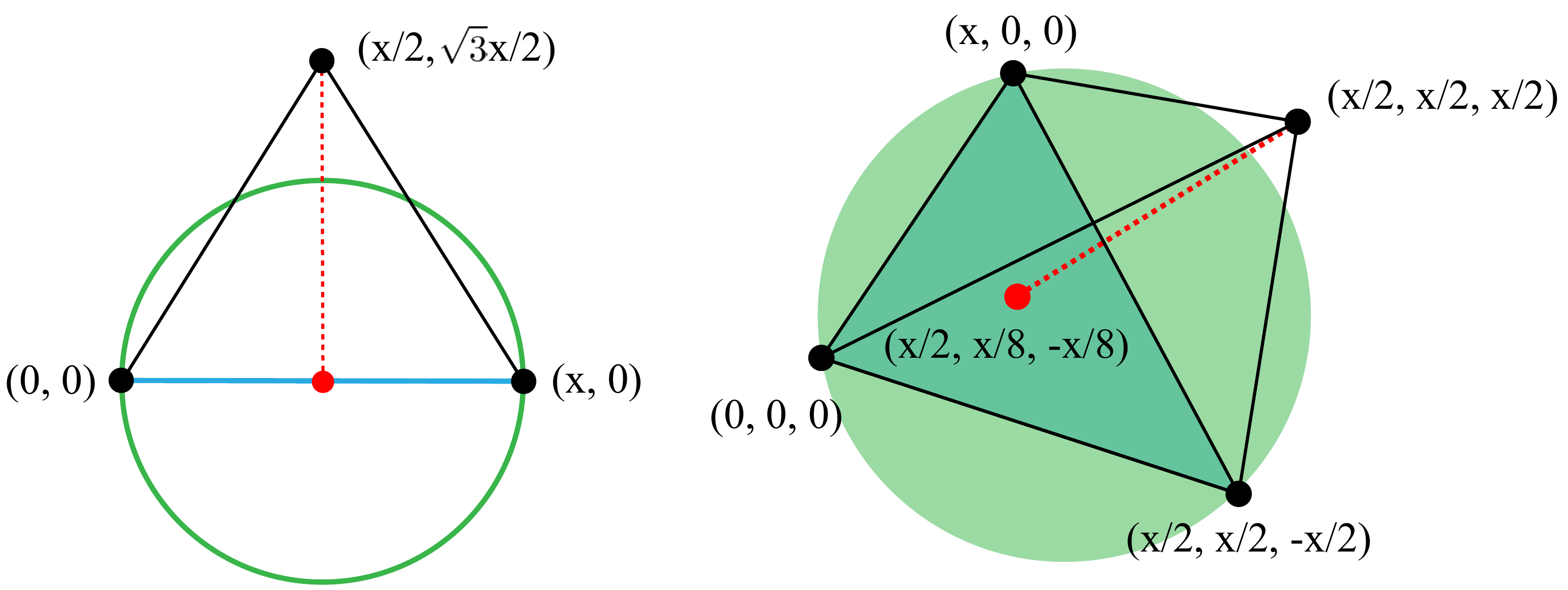}
    \caption{\textbf{Common signed distance for a 2D (left) and 3D (right) face in (initial) regular grid.} We compute the signed distance by subtracting the radius of the \textcolor{LimeGreen}{minimum bounding ball} from the length of the \textcolor{red}{red line}. The \textcolor{red}{red dot} represents the center of the \textcolor{LimeGreen}{minimum bounding ball}.}
    \label{fig:sigmoid-alpha}
\end{figure}

The sigmoid coefficient $\alpha_{min}$ plays a critical role in determining the probability to which a signed distance $d$ is mapped. Even if a face $F$ satisfies the \emph{Minimum-Ball} condition by a large margin ($d(B_{F}, \mathbb{F}) \gg 0$), indicating a high existence probability for $F$, a small $\alpha_{min}$ value would result in the derived probability being only slightly greater than $0.5$. To minimize such mismatches, we set $\alpha_{min}$ based on the density of the grid from which optimization begins.

As discussed in Appendix~\ref{appendix:pipeline-1-init}, the reconstruction process often starts from a fixed triangular (2D) or tetrahedral (3D) grid (\cref{fig:grid}). At the initial state, every face in the grid satisfies the \emph{Minimum-Ball} condition (Appendix~\ref{appendix:pipeline-1-init}). Notably, every interior face in the grid shares a common signed distance $d_{common} > 0$. Let us denote $x$ as the edge length of the grid, applicable for both 2D and 3D cases. Then, the common signed distance can be computed as follows:

For $d = 2$:
\begin{equation}
    d_{common} = \frac{\sqrt{3} - 1}{2}x.
\end{equation}

For $d = 3$:
\begin{equation}
    d_{common} = \frac{\sqrt{34} - 3\sqrt{2}}{8}x.
\end{equation}

In~\cref{fig:sigmoid-alpha}, we provide an illustration of the reasoning behind these results. By calculating these common signed distances, we use them to determine $\alpha_{min}$. Specifically, during the first epoch, we set $\alpha_{min} = 32 / d_{common}$, ensuring that the probability for every face in the grid is initialized to $\sigma(32) \simeq 1.0$. 

In subsequent epochs, $\alpha_{min}$ is adjusted to account for the additional points introduced during subdivision. If $\alpha_{min}^{1}$ represents the value in the first epoch, the value for the $i$-th epoch is given by:
\begin{equation}
    \alpha_{min}^{i} = \frac{\alpha_{min}^{1}}{2^{i - 1}}.
\end{equation}

\subsection{Nearest neighbor caching}
\label{appendix:nn-caching}

In~\cref{sec:min-ball,sec:exp-tess-speed}, we demonstrated how the \emph{Minimum-Ball} algorithm significantly accelerates tessellation. This process can be further optimized by periodically caching the $K$-nearest neighbors of each $B_{F}^{c}$ in $\mathbb{P}$ and using the cached neighbors for computing probabilities until the next cache update. This optimization is feasible because the $K$-nearest neighbors generally do not change significantly during the optimization process.

Let us define the number of optimization steps as $n_0$ and the number of steps between cache updates as $n_1$. At every $n_1$ steps, we refresh the query faces $\mathbb{F}$ based on the current set of points $\mathbb{P}$ and recompute the centers of the minimum bounding balls for the query faces ($B_{\mathbb{F}}^{c}$). Then, we identify the $K$-nearest neighbors of $B_{\mathbb{F}}^{c}$ in $\mathbb{P}$. In practice, we compute the $(K + d)$-nearest neighbors instead, as explained in Appendix~\ref{appendix:min-ball-algo}, to ensure robustness. 

During the subsequent optimization steps, for a given face $F$, we compute the distance from $B_{F}^{c}$ to the cached $K$-nearest neighbors in $\mathbb{P}$ and select the nearest neighbor from the cache to compute the signed distance $d(B_{F}, \mathbb{P})$. This mechanism is described in detail in~\cref{alg:pos-optim} and Appendix~\ref{appendix:pipeline-2-pos-optim} in the context of point position optimization. In our experiments for 3D multi-view reconstruction, we set $n_0 = 2000$, $n_1 = 50$, and $K = 10$.


\begin{figure}
    \centering
    \includegraphics[width=1.0\linewidth]{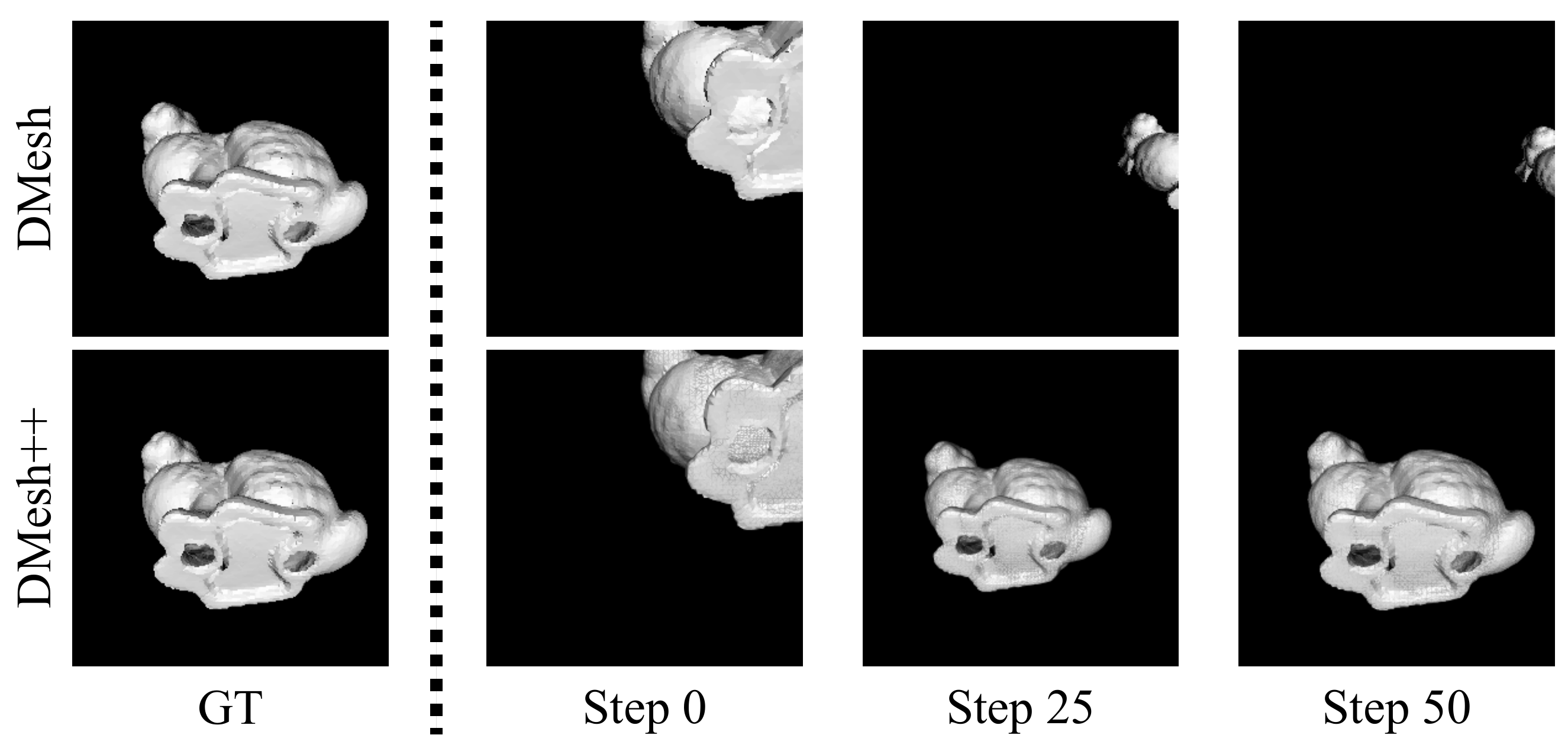}
    \caption{\textbf{Role of visibility gradient in geometric optimization.} In this experiment, we optimize the translation vector of the object by comparing its rendered image and the ground truth image on the left. Since the differentiable renderer of DMesh~\citep{son2024dmesh} does not implement visibility gradient, while ours does, DMesh fails to find the correct translation vector.}
    \label{fig:aa-test}
    \vspace{-1em}
\end{figure}

\section{Details about Reconstruction Process}
\label{appendix:recon-detail}

In this section, we provide implementation details about our reconstruction process described in~\cref{sec:recon-process}. Before delving into these details, we introduce the loss formulations for reconstruction problems.

\subsection{Loss Formulation}
\label{appendix:pipeline-loss}

Our final loss, $L$, is comprised of main reconstruction loss ($L_{recon}$) and two regularization terms: $L_{qual}$ and $L_{real}$.
\begin{equation}
    L = L_{recon} + \lambda_{qual}\cdot L_{qual} + \lambda_{real}\cdot L_{real}.
\label{eq:final-loss}
\end{equation}

We explain each of these terms below.

\subsubsection{Reconstruction Loss ($L_{recon}$)}

Reconstruction loss drives the reconstruction process by comparing our current probabilistic mesh and the given ground truth observations. For different observations, we need different loss functions as follows.

\paragraph{Point Cloud} 
When ground truth point clouds are provided, we utilize the expected Chamfer Distance (CD) proposed by~\citep{son2024dmesh}. In this formulation, when sampling points from our mesh, we assign an existence probability to each sampled point, which matches the probability of the face from which the point is sampled. The expected CD incorporates these probabilities, unlike the traditional Chamfer Distance, which does not. For further details, refer to~\citep{son2024dmesh}.

\paragraph{Multi-view Images} 
For rendering probabilistic faces, we interpret each face's existence probability as its opacity, following~\citep{son2024dmesh}. To render large number of semi-transparent faces efficiently, we use the differentiable renderer of~\citep{son2024dmesh}, but we found out that it does not implement visibility gradient that is necessary for optimizing geometric properties (\cref{fig:aa-test}). For the details about this visibility gradient, please refer to~\citep{laine2020modular}, which implemented the visibility gradient using anti-aliasing. Following their path, we enhanced the differentiable renderer of DMesh by implementing anti-aliasing in CUDA, which provides us visibility gradients.

In~\cref{fig:aa-impl}, we briefly illustrate how we implemented anti-aliasing in CUDA. Specifically, for each (triangular) face-pixel pair $(F, P)$, we project $F$ onto the image space and compute the overlapping area $A(F, P)$ between the projected triangle and the pixel (\cref{fig:aa-impl} right, blue area). Denoting the total area of the pixel as $A(P)$, the ratio of the overlapping area in the given pixel, $\rho(F, P)$, is computed as:
\begin{equation}
    \rho(F, P) = \frac{A(F, P)}{A(P)}.
\end{equation}

\begin{figure}[t]
    \centering
    \includegraphics[width=1.0\linewidth]{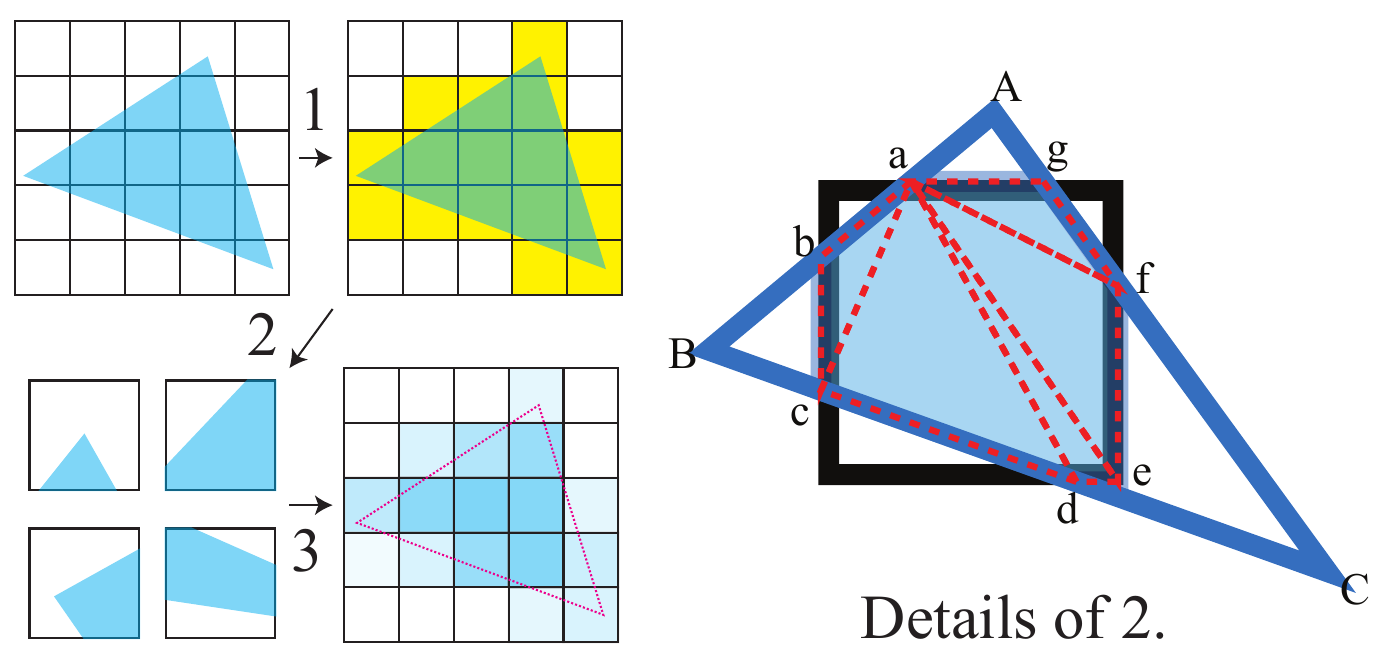}
    \caption{\textbf{Implementation of anti-aliasing in our differentiable renderer.} On the left, we show the process of anti-aliasing: 1) Find pixels that overlap with the given triangle, 2) Find the area that each pixel overlaps with the given triangle, and 3) Determine the color of each pixel based on the overlap area. On the right, we show details of the step 2.}
    \label{fig:aa-impl}
    \vspace{-1em}
\end{figure}

We use $\rho(F, P)$ to determine the opacity of the face $F$ at the pixel $P$. If the opacity of $F$ is $\alpha(F)$, we compute the face opacity at pixel $P$, $\alpha(F, P)$, as:
\begin{equation}
    \alpha(F, P) = \alpha(F) \cdot \rho(F, P) \le \alpha(F).
\end{equation}

Thus, the opacity of $F$ at $P$ is proportional to the overlapping area between the triangle and the pixel.

In the right figure of~\cref{fig:aa-impl}, we illustrate the process of computing the overlapping area $A(F, P)$. The vertices of $F$ are projected onto the image plane and visited in counterclockwise order (e.g., A - B - C - A in the illustration). We then find the intersection points between the triangle edges and the pixel boundaries. These intersection points form the vertices of the (convex) overlapping polygon.

For example, vertices (a) and (b) are found by calculating the intersections of $\overline{AB}$ with the pixel boundaries. The vertices of the overlapping polygon are stored in counterclockwise order, and the polygon is subdivided into a set of sub-triangles, as shown by the dotted red lines in the visualization. The total area of the overlapping polygon is obtained by summing the areas of the sub-triangles.

Using this enhanced differentiable renderer, we render multi-view images and compute the $L_1$ loss between the ground truth images as the reconstruction loss, $L_{recon}$. 

\subsubsection{Triangle Quality Loss ($L_{qual}$)}

To improve the triangle quality of the final mesh, we adopt the triangle quality loss ($L_{qual}$) of DMesh~\citep{son2024dmesh}. Specifically, the loss is defined as follows:
\begin{equation}
    L_{qual} = \frac{1}{|\mathbb{F}|}\sum_{F \in \mathbb{F}} AR(F)\cdot \Lambda(F),
\end{equation}

where $\mathbb{F}$ is the set of every face combination we consider, $AR(\cdot)$ is a function that computes aspect ratio of the given face, and $\Lambda(\cdot)$ is the face probability function defined in~\cref{sec:preliminary}.

\subsubsection{Real Loss ($L_{real}$)}

We minimize the sum of point-wise real values ($\psi)$, so that we can remove redundant faces as much as possible during optimization. The loss $L_{real}$ is simply defined as:
\begin{equation}
    L_{real} = \frac{1}{|\mathbb{P}|}\sum_{p \in \mathbb{P}} \Psi(p),
\end{equation}

where $\Psi(\cdot)$ is the function that returns the real value of the given point, as defined in~\cref{sec:preliminary}.

\subsection{Reconstruction Steps}

Now we provide detailed explanations about each step in the reconstruction process.

\subsubsection{Step 1: Initialization}
\label{appendix:pipeline-1-init}

We initialize our point features differently for two different scenarios: when sample point cloud is given or not. 

\paragraph{Point Cloud Init.} 

When a point cloud sampled from the target shape is given, we can initialize our point features using the point cloud, so that the initial configuration already captures the overall structure of the target shape (\cref{fig:tsne}). Specifically, for the given point cloud, we estimate the density of the point cloud by computing the distance to the nearest point for each point. Then, we down sample the point cloud using a voxel grid, of which size is defined as the point cloud density, to remove redundant points and holes. Finally, we follow the initialization scheme of DMesh~\citep{son2024dmesh} using the down sampled point cloud.

\paragraph{Regular Grid Init.}

If we do not have any prior knowledge about the target shape, we first organize the points in a regular grid, ensuring that every face in the grid satisfies the \emph{Minimum-Ball} condition (Definition~\ref{def:min-ball}), and initialize the point-wise real values ($\psi$) with additional features (\eg colors). This regular grid guarantees that the faces observed in this step will also be observable in the subsequent step, where the \emph{Minimum-Ball} algorithm determines face existence. For $d = 2$, this condition is satisfied by forming every triangle in the grid as an equilateral triangle. For $d = 3$, we use a body-centered cubic lattice. The grids are illustrated in~\cref{fig:grid}.

With these fixed points and faces, we formulate the final loss by setting $\lambda_{qual} = 0$ and $\lambda_{real}=10^{-4}$ in~\cref{eq:final-loss}, and minimize it to determine which faces to include in the final mesh. After optimization, we collect points with real values larger than $0.01$ to ensure that as many faces as possible are available for the next optimization step, thereby reducing the risk of holes in the surface. 

\begin{figure}[t]
    \centering
    \includegraphics[width=1.0\linewidth]{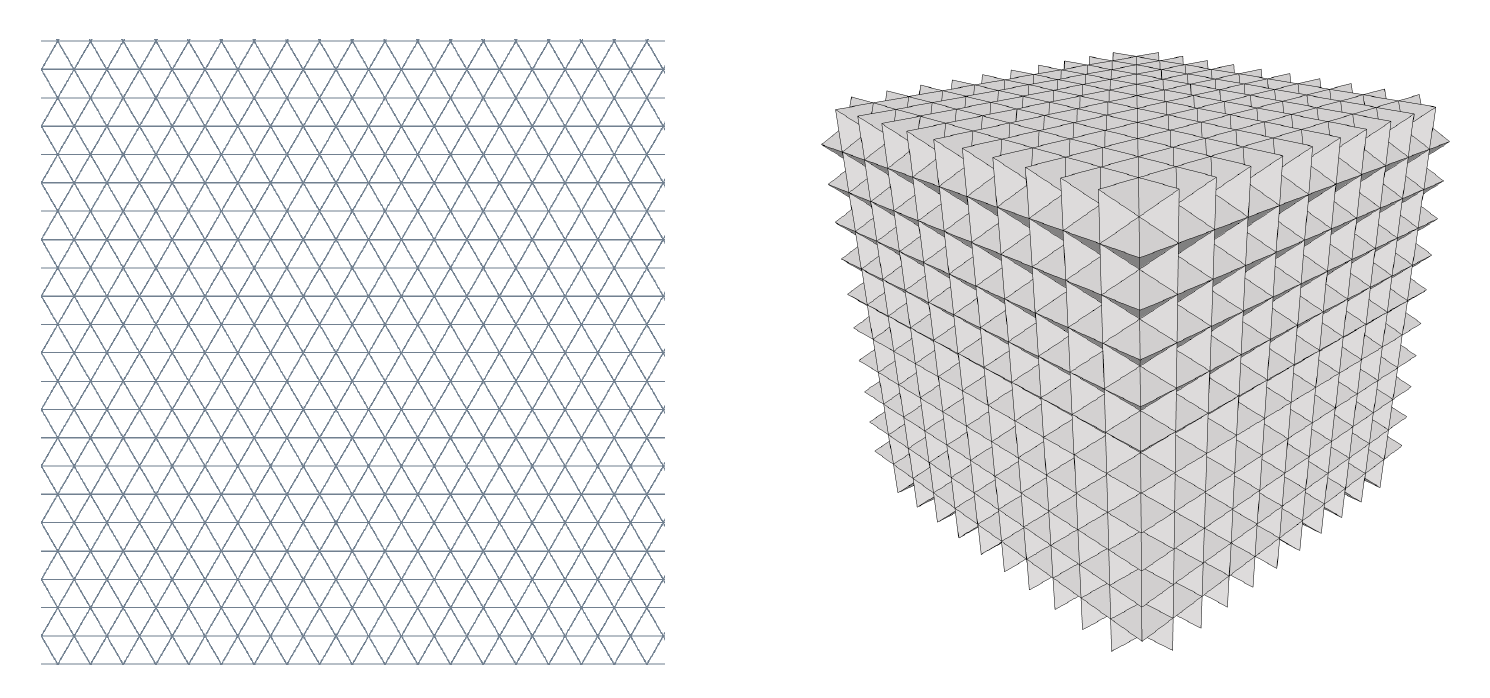}
    \caption{\textbf{Grid structure to initialize real values in 2D (left) and 3D (right).} Every face in the grid structure satisfies \emph{Minimum-Ball} condition (Definition~\ref{def:min-ball}).}
    \label{fig:grid}
\vspace{-1em}
\end{figure}

\subsubsection{Step 2: Position Optimization}
\label{appendix:pipeline-2-pos-optim}

In this step, we fix the point-wise real values ($\psi$) and optimize only the point positions. For clarity, we formally describe the process in~\cref{alg:pos-optim}, and explain the algorithm line by line below:

\begin{itemize}
    \item \textbf{Lines 1-2:} For the given set of points, we denote their positions as $\mathbb{P}$ and their real values as $\Psi$.
    \item \textbf{Line 3:} We define the total number of optimization steps as $n_0$.
    \item \textbf{Line 4:} We define the number of optimization steps required to refresh query faces and their nearest neighbor cache as $n_1$. Since point positions are optimized, the point configuration evolves during optimization, potentially leading to the emergence of new faces that were previously unobservable. To account for these changes, we refresh the query faces periodically.
    \item \textbf{Line 5:} We denote the number of nearest neighbors to store in the cache for the query faces as $K$.
    \item \textbf{Lines 6-7:} The optimization process runs for $n_0$ steps.
    \item \textbf{Lines 8-12:} At every $n_1$ step, we update the query faces based on the current point configuration. 

    In the \texttt{Update-Query-Faces} function, which uses point positions and their real values, we:
    \begin{itemize}
        \item Extract points with a real value of $1$.
        \item For each extracted point, find its 10-nearest neighbors that also have a real value of $1$, since any face containing a point with a real value of $0$ is considered non-existent.
        \item Perform Delaunay Triangulation (DT) for the entire point set and collect faces in DT where all points have a real value of $1$. This ensures the inclusion of as many faces as possible during optimization, helping to eliminate potential holes later.
    \end{itemize}

    For the updated query faces, we compute the centers of their minimum bounding balls. Subsequently, we identify the $K$-nearest neighbors of these centers in $\mathbb{P}$ and store this information in the nearest neighbor cache $\mathbb{C}$.

\begin{algorithm}[t]
\caption{\texttt{Position Optimization}}
\label{alg:pos-optim}
\begin{algorithmic}[1]
    \State $\mathbb{P}, \Psi \gets$ Set of points and their real values
    \State $\alpha_{min} \gets$ Coefficient for sigmoid function
    \State $n_{0} \gets$ Number of optimization steps
    \State $n_{1} \gets$ Number of refresh steps for query faces
    \State $K \gets$ Number of nearest neighbors to store in cache
    \State $i \gets 0$
    \While{$i < n_0$}
        \If{$i \mod n_1 = 0$}
            \State $\mathbb{F} \gets$ \texttt{Update-Query-Faces($\mathbb{P}, \Psi$)}
            \State $B_{\mathbb{F}}^{c}, B_{\mathbb{F}}^{r} \gets$ \texttt{Compute-Minimum-Ball}($\mathbb{P, F}$)
            \State $\mathbb{C} \gets$ \texttt{Find-KNN}($B_{\mathbb{F}}^{c}, \mathbb{P}$, $K$)
        \EndIf
        \State $B_{\mathbb{F}}^{c}, B_{\mathbb{F}}^{r} \gets$ \texttt{Compute-Minimum-Ball}($\mathbb{P, F}$)
        \State $P^{nearest}_{\mathbb{F}} \gets$ \texttt{Find-NN-CACHE}($B_{\mathbb{F}}^{c}, \mathbb{C}$)
        \State $d(B_{\mathbb{F}}, \mathbb{P}) \gets B_{\mathbb{F}}^{r} - ||P^{nearest}_{\mathbb{F}} - B^{c}_{\mathbb{F}}||$
        \State $\lambda_{min}(\mathbb{F}) \gets \sigma(d(B_{\mathbb{F}}, \mathbb{P}) \cdot \alpha_{min})$
        \State $\lambda(\mathbb{F}) \gets \lambda_{min}(\mathbb{F})$
        \State $L \gets$ \texttt{Compute-Loss}($\mathbb{P, F}, \lambda(\mathbb{F})$)
        \State Update $\mathbb{P}$ to minimize $L$
        \State $i \gets i + 1$
    \EndWhile
\end{algorithmic}
\end{algorithm}

    \item \textbf{Lines 13-16:} Using the current point configuration, we compute the minimum bounding balls ($B_{\mathbb{F}}$) for the query faces. For each bounding ball center, we find the nearest neighbor in the nearest neighbor cache $\mathbb{C}$ by calculating the distances to points in $\mathbb{C}$ and selecting the closest one. We then compute the signed distance $d(B_{\mathbb{F}}, \mathbb{P})$ for the query faces and use it to get the probability $\lambda_{min}(\mathbb{F})$.

    \item \textbf{Line 17:} Since the query faces consist only of points with a real value of $1$, we set the final face probability $\lambda(\mathbb{F})$ to be the same as $\lambda_{min}(\mathbb{F})$ (\cref{sec:preliminary}).

    \item \textbf{Line 18:} Based on the point positions, query faces, and their existence probabilities, we compute the loss $L$ to minimize following~\cref{eq:final-loss}. 
    
    \item \textbf{Lines 19-20:} Finally, we update the point positions $\mathbb{P}$ to minimize $L$ and iterate the process.
\end{itemize}

\subsubsection{Step 3: Real Value Optimization}
\label{appendix:pipeline-4-real-optim}

In this step, we re-optimize the point-wise real values while keeping the point positions fixed. From the current point configuration, we identify all faces in the Delaunay Triangulation (DT) that satisfy the \emph{Minimum-Ball} condition. Note that any face satisfying this condition must exist in the DT (Lemma~\ref{lemma:min-ball-dt}). Thus, we first compute the DT of the points and then verify whether each face in the DT satisfies the \emph{Minimum-Ball} condition.

Next, we follow a similar optimization process to Step 1 (Appendix~\ref{appendix:pipeline-1-init}). Additionally, if it was the multi-view reconstruction task, we remove invisible faces to remove redundant faces as much as possible. If this was the last epoch, we return the post-processed mesh. 

\begin{figure}[t]
    \centering
    \includegraphics[width=\linewidth]{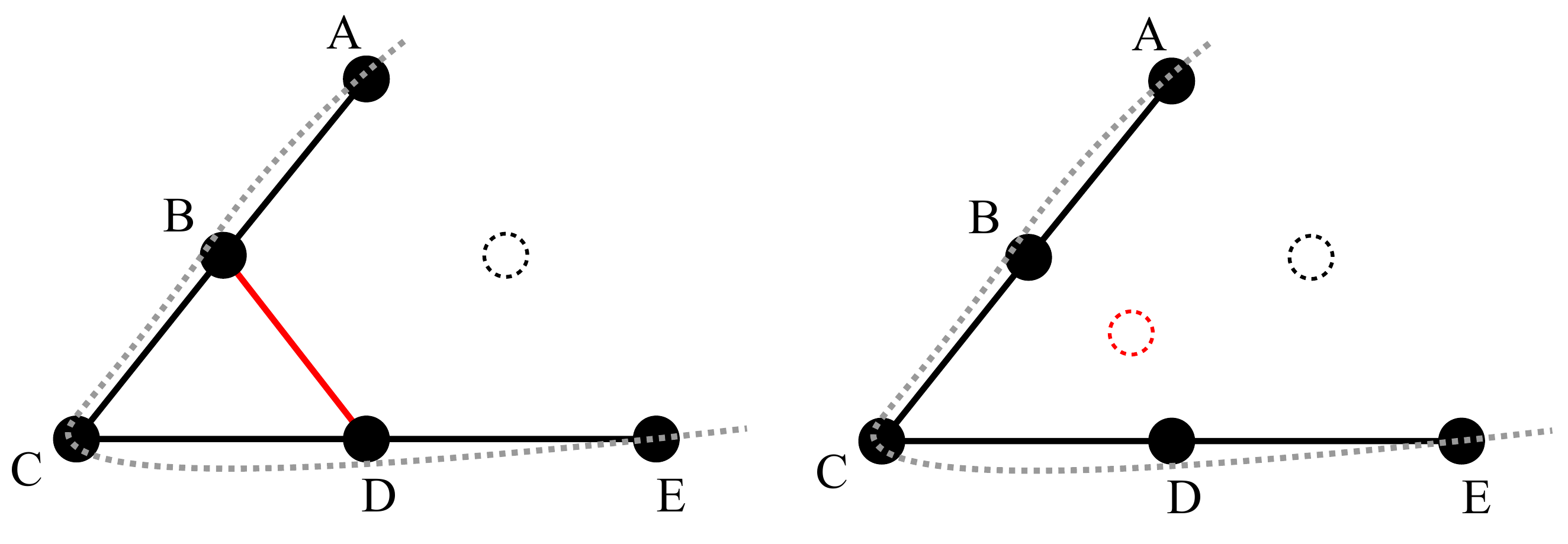}
    \caption{\textbf{Point insertion for removing undesirable face.} (Left) To reconstruct the \textcolor{gray}{ground truth shape}, we need to set the real value ($\psi$) of points A-E to 1. The point rendered with dotted line has real value of 0. Then, we observe unnecessary face \textcolor{red}{$\overline{BD}$} exists. (Right) To remove this face, we insert \textcolor{red}{additional point} that carries $0$ real value near the unnecessary face.}
    \label{fig:subdiv-ambiguity}
    \vspace{-1em}
\end{figure}

\subsubsection{Step 4: Subdivision}

\begin{wrapfigure}{r}{0.25\linewidth}
\vspace{-1em}
\hspace{-1em}
    \includegraphics[width=\linewidth]{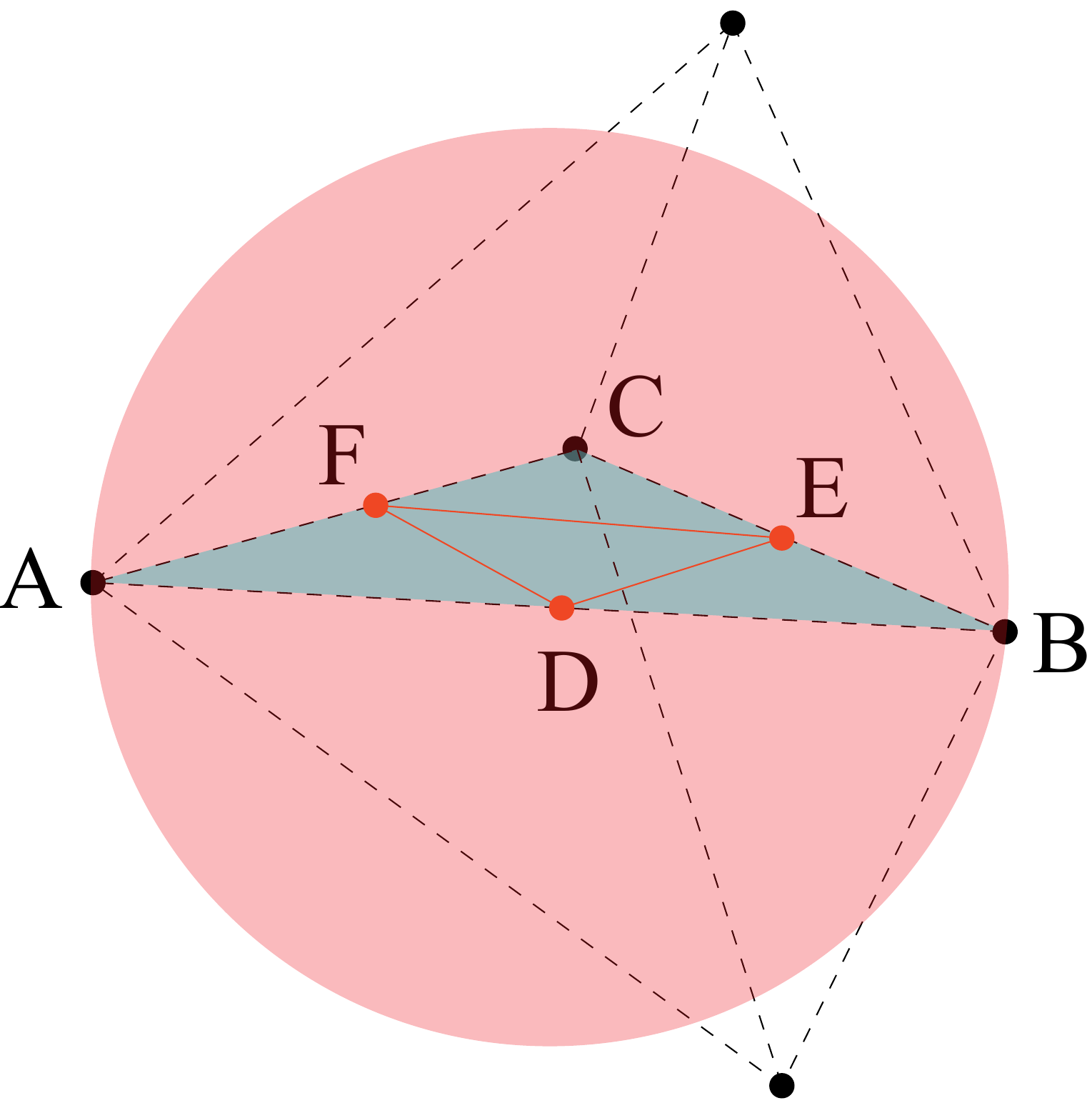}
    \label{fig:enter-label}
\vspace{-1em}
\end{wrapfigure}

To reconstruct fine geometric details of the given shape, we subdivide the current mesh mainly by adding points with $\psi = 1$ at the middle of edges that are adjacent to currently existing faces. In the inset, points $E, D, F$ are the newly inserted points. They form 4 sub-faces with $A, B, C$, and they all satisfy Definition~\ref{def:min-ball}. Therefore, we can guarantee that these sub-faces will exist at the start of next epoch. Note that this guarantee does not hold for WDT.

At the same time, it is also possible to insert additional points into faces that \emph{should not} exist in the next epoch, effectively removing such faces at the start of the next epoch. For example, during the real value optimization step in the pipeline (\cref{fig:pipeline}, Appendix~\ref{appendix:pipeline-4-real-optim}), invisible faces are removed for multi-view reconstruction task. After optimization, we may observe removed faces with all their points have a real value of $1.0$, creating a contradiction. This situation could arise due to ambiguities in the mesh definition, as illustrated in~\cref{fig:subdiv-ambiguity}.

To eliminate these undesirable faces, additional points with a real value of $0$ are inserted at their circumcenters. Consequently, after subdivision, several holes may appear on the surface because these additional points might also be inserted into faces that \emph{should} exist (\cref{fig:pipeline}). However, most of these holes are resolved during subsequent optimization steps.

\section{Experimental Details and Additional Results}
\label{appendix:exp-details}

In this section, we outline the experimental settings used for the results in~\cref{sec:experiments} and present additional results to support our claims.

\subsection{Dataset}

Here, we provide details on the datasets described in~\cref{sec:exp-recon-task}.

\subsubsection{Font}

We used four font styles: Pacifico, Permanent-Marker, Playfair-Display, and Roboto.

\subsubsection{Thingi10K}

We manually selected 10 closed surfaces and 10 open surfaces from the Thingi10K dataset~\citep{zhou2016thingi10k}. Specifically, we used the following models, denoted by their file IDs:

\begin{itemize}
    \item \textbf{Closed surfaces:} 47926, 68380, 75147, 80650, 98576, 101582, 135730, 274379, 331105, and 372055.
    \item \textbf{Open surfaces:} 40009, 41909, 73058, 82541, 85538, 131487, 75846, 76278, 73421, and 106619.
\end{itemize}

These models were chosen because they exhibit minimal self-occlusions, enabling dense observations from multi-view images. Additionally, we randomly selected 500 models and used for comparisons.

\subsubsection{Objaverse}

We manually selected 30 mesh models that exhibit diverse topology from Objaverse~\citep{deitke2023objaverse}, which include both closed and open surfaces, and also have small scenes. Some of these models are rendered in~\cref{fig:tsne,fig:pipeline,fig:3d-pc-qual,fig:3d-mv-qual-closed,fig:3d-pc-qual-closed}.

\subsection{2D Point Cloud Reconstruction}
\label{appendix:2d-exp-details}

\subsubsection{Hyperparameters}

\paragraph{DMesh++}

\begin{itemize}
    \item Learning rate (real value, $\psi$): $0.3$
    \item Learning rate (position): $0.001$
    \item Number of epochs: 1
    \item Number of optimization steps
    \begin{itemize}
        \item Step 1 (Real value initialization): 100
        \item Step 2 (Point position optimization): 500
    \end{itemize}
\end{itemize}

\subsubsection{Reconstruction of Complex Drawings}
\label{appendix:2d-drawing-recon}

In~\cref{fig:supp-2d-drawing}, we present the reconstruction results for complex 2D drawings. As the figure illustrates, DMesh++ successfully reconstructs intricate 2D geometries from point clouds, even when the number of edges approaches nearly 1 million.

\begin{figure*}
    \centering
    \includegraphics[width=\linewidth]{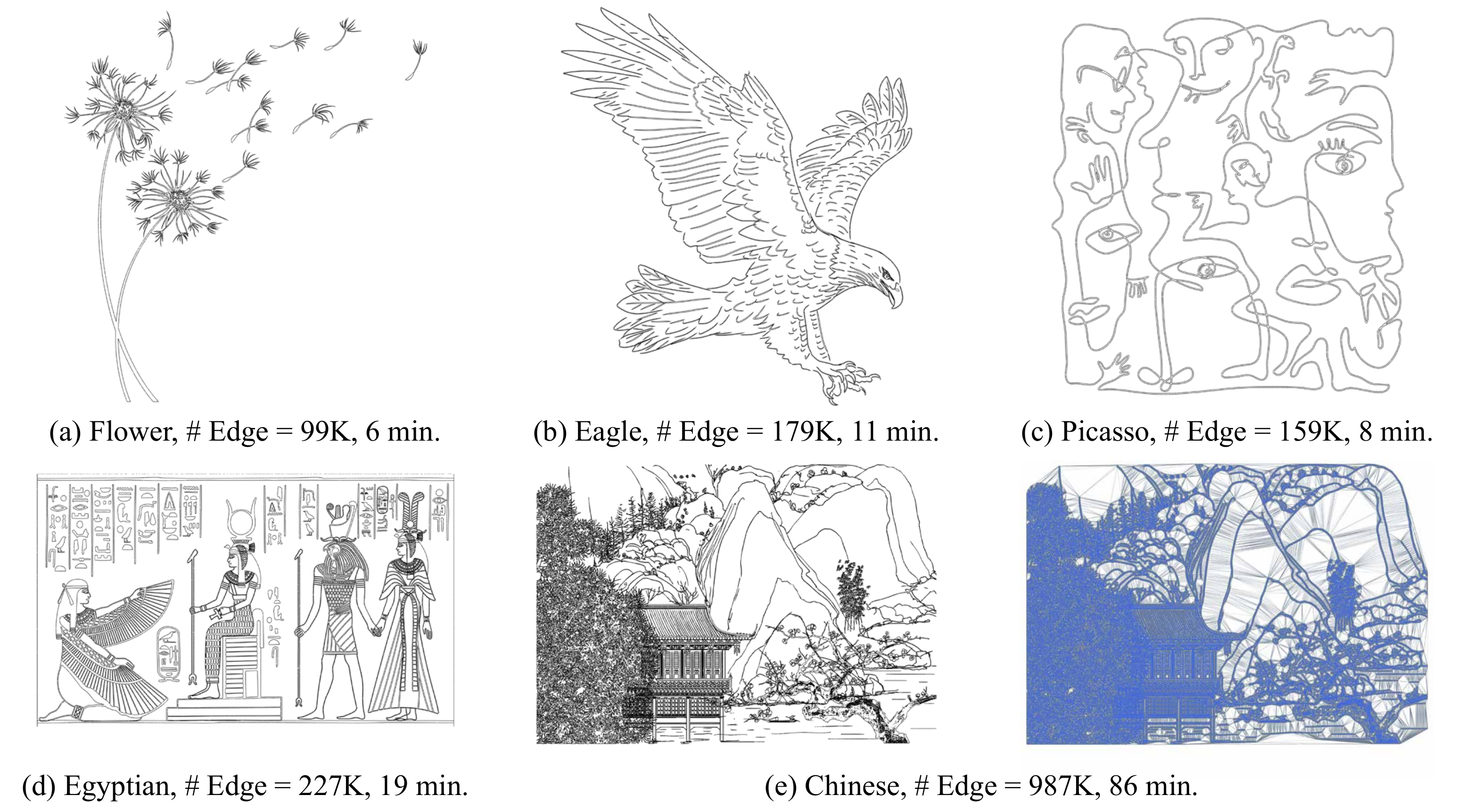}
    \caption{\textbf{2D point cloud reconstruction result for complex drawings.} For each drawing, we report both the number of edges and the reconstruction time. For the Chinese drawing, we additionally render the ``imaginary'' part on the right to clearly illustrate its complexity.}
    \label{fig:supp-2d-drawing}
    \vspace{-1.0em}
\end{figure*}

\subsection{3D Point Cloud Reconstruction}
\label{appendix:3d-pc-recon-details}

\subsubsection{Hyperparameters}

\paragraph{DMesh++}

\begin{itemize}
    \item Initial Grid Edge Length: $3\times$ input point cloud density
    \item Learning rate (position): 0.001
    \item Number of epochs: 1
    \item Number of optimization steps
    \begin{itemize}
        \item Step 2 (Point position optimization): 2000
        \item Step 3 (Real value optimization): 0
    \end{itemize}
\end{itemize}

\subsection{3D Multi-View Reconstruction}
\label{appendix:3d-exp-details}

\begin{figure*}
    \centering
    \includegraphics[width=\linewidth]{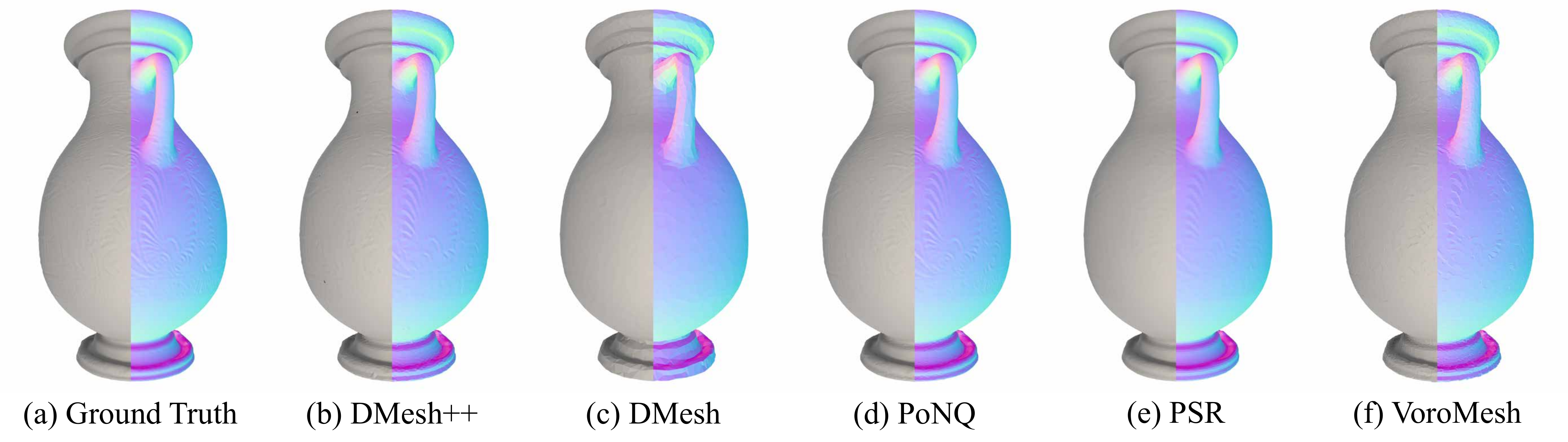}
    \caption{\textbf{Qualitative comparison of 3D point cloud reconstruction results for a closed surface (vase).} For each image, we render the view-point normal on the right, and the diffuse image on the left. Among the baseline methods that reconstruct watertight mesh from point clouds, PoNQ~\citep{maruani2024ponq} performs the best in reconstructing fine geometric details. While DMesh~\citep{son2024dmesh} fails at reconstructing such details due to the lack of mesh complexity, DMesh++ successfully recovers them and produce comparable result to PoNQ.}
    \label{fig:3d-pc-qual-closed}
\end{figure*}

\begin{figure*}
    \centering
    \includegraphics[width=\linewidth]{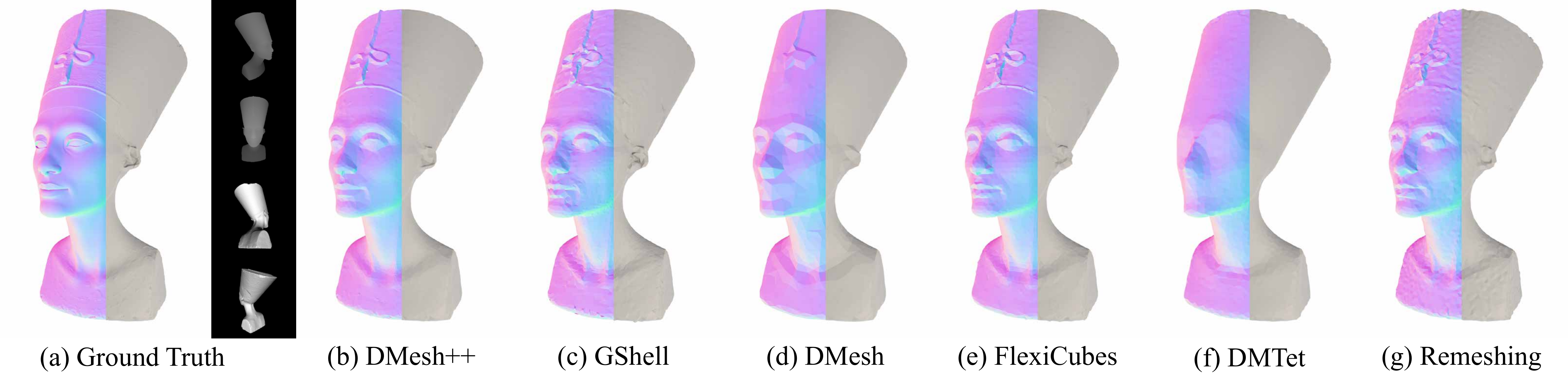}
    \caption{\textbf{Qualitative comparison of 3D multi-view reconstruction results for a closed surface (sculpture).} We render the input diffuse and depth images alongside the ground truth image. For each image, we render the view-point normal on the left, and the diffuse image on the right. We can observe that the reconstruction result of DMesh++ is as good as the other baseline methods that are optimized for closed surfaces.}
    \label{fig:3d-mv-qual-closed}
    \vspace{0em}
\end{figure*}

\subsubsection{Hyperparameters}

\noindent\textbf{Remeshing~\citep{palfinger2022continuous}}
\begin{itemize}
    \item Image Batch Size: 8
    \item Number of Optimization Steps: 1000
    \item Learning Rate: 0.1
    \item Edge Length Limits: [0.02, 0.15]
\end{itemize}

The "Edge Length Limits" were adjusted to produce meshes with a similar number of vertices and faces to other methods for a fair comparison.

\vspace{4pt}\noindent\textbf{DMTet~\citep{shen2021deep}}
\begin{itemize}
    \item Image Batch Size: 8
    \item Number of Optimization Steps: 5000
    \item Learning Rate: 0.001
    \item Grid Resolution: 128
\end{itemize}

The SDF was initialized to a sphere, as in the original implementation, before starting optimization.

\vspace{4pt}\noindent\textbf{FlexiCubes~\citep{shen2023flexible}}
\begin{itemize}
    \item Image Batch Size: 8
    \item Number of Optimization Steps: 2000
    \item Number of Warm-up Steps: 1500
    \item Learning Rate: 0.01
    \item Grid Resolution: 80
    \item Triangle Aspect Ratio Loss Weight: 0.01
\end{itemize}

The SDF was initialized randomly, following the original implementation. To improve the quality of the output mesh, we adopted a triangle aspect ratio loss, designed to minimize the average aspect ratio of triangles in the mesh. The mesh was first optimized for 1500 steps as a warm-up without the triangle aspect ratio loss, followed by 500 steps with the additional loss. 

Additionally, we observed that the output mesh often included false internal structures, which significantly degraded the Chamfer Distance (CD) compared to the ground truth mesh. To mitigate this, we performed a visibility test on the output mesh to remove these false internal structures as much as possible.

\vspace{4pt}\noindent\textbf{GShell~\citep{liu2023ghost}}

\begin{itemize}
    \item Image Batch Size: 8
    \item Number of Optimization Steps: 5000
    \item Number of Warm-up Steps: 4500
    \item Learning Rate: 0.01
    \item Grid Resolution: 80
    \item Triangle Aspect Ratio Loss Weight: 0.0001
\end{itemize}

To enhance the quality of the output mesh, we employed the same additional measures as FlexiCubes. We found that longer optimization steps were required for GShell compared to FlexiCubes to effectively handle open surfaces.

\vspace{4pt}\noindent\textbf{DMesh++ Settings}

\begin{itemize}
    \item Initial Grid Edge Length: 0.05
    \item Learning Rate (Real Value, $\psi$): 0.01
    \item Learning Rate (Position): 0.001
    \item Number of Epochs: 2
    \begin{itemize}
        \item Image Res. / Batch Size at Epoch 1: (256, 256), 1
        \item Image Res. / Batch Size at Epoch 2: (512, 512), 1
    \end{itemize}
    \item Number of Optimization Steps
    \begin{itemize}
        \item Step 1 (Real Value Initialization): 1000
        \item Step 2 (Point Position Optimization): 2000
        \item Step 3 (Real Value Optimization): 1000
    \end{itemize}
\end{itemize}

In the first epoch, we used lower-resolution images as part of a coarse-to-fine approach.

\subsubsection{Limitations}
\label{appendix:mv-limit}

Despite DMesh++'s success in reconstructing geometrically accurate meshes from multi-view images (of a synthetic object or scene), it currently cannot recover meshes from real-world images. This limitation stems from an inadequate rendering model — our current per-vertex color model is too simple to capture detailed geometry, and our algorithm assumes full knowledge of lighting conditions, which is not available in real-world scenarios.

\begin{figure}
    \centering
    \includegraphics[width=\linewidth]{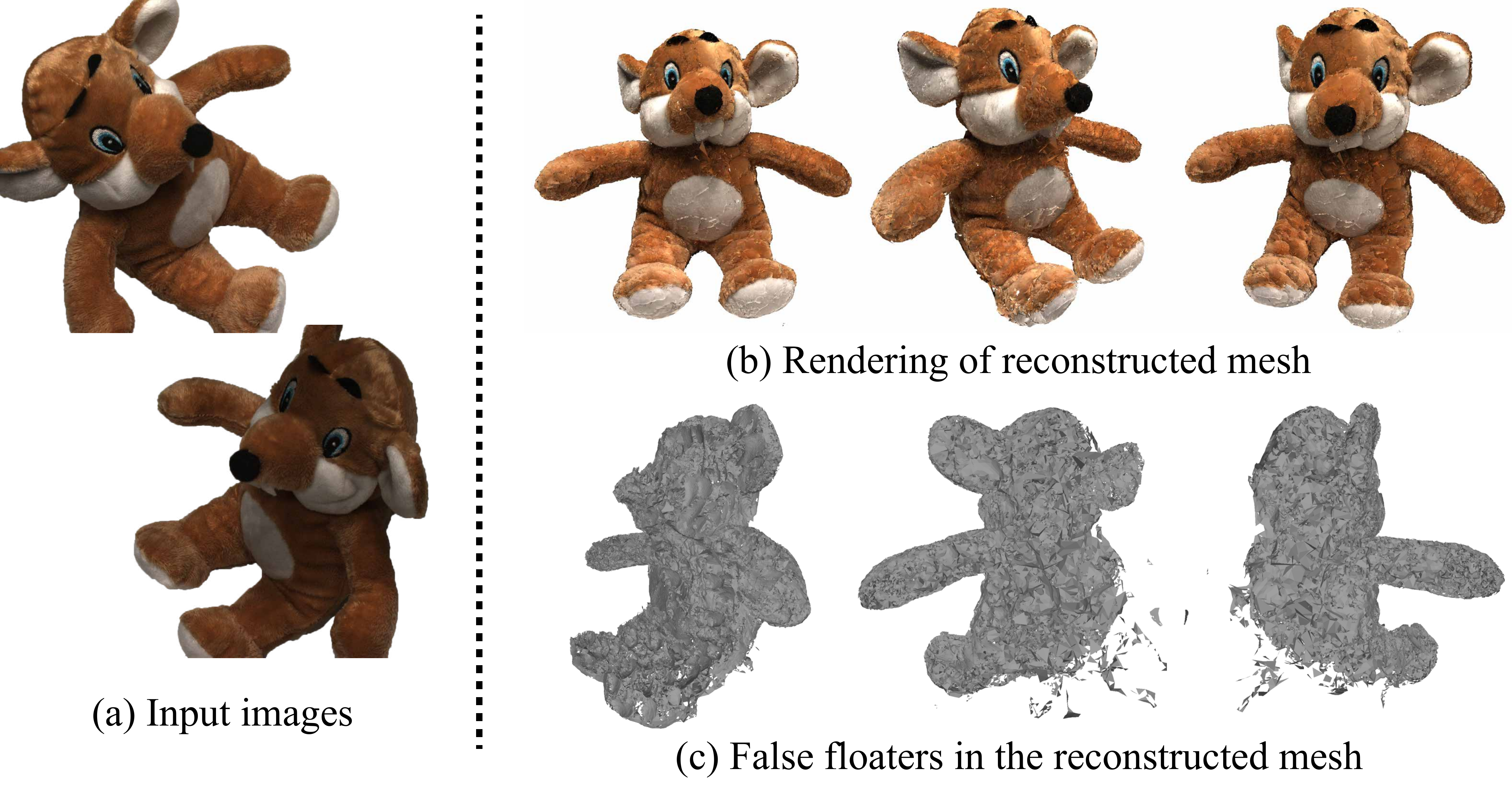}
    \caption{\textbf{3D reconstruction from real-world images in DTU dataset~\citep{jensen2014large}.} The input images are shown on left, and the reconstructed mesh is shown on right.}
    \label{fig:dtu-result}
    \vspace{-1em}
\end{figure}

To illustrate this, we applied our reconstruction algorithm to real-world images from the DTU dataset~\citep{jensen2014large}, as shown in~\cref{fig:dtu-result}. While our method approximates the real-world images (\cref{fig:dtu-result}(b)), the extracted mesh exhibits numerous false floaters (\cref{fig:dtu-result}(c)). We believe this suboptimal result is due to the lack of proper rendering models and regularizations, and addressing this issue by integrating DMesh++ with other reconstruction mechanisms is an exciting direction for future research, as discussed in~\cref{sec:discussion}.

\section{\textit{Reinforce-Ball} algorithm}
\label{appendix:rl-ball}

Here we introduce an experimental algorithm that further enhances DMesh++'s capability. As discussed in~\cref{sec:preliminary}, DMesh++ no longer uses the per-point weights found in DMesh~\citep{son2024dmesh}. In DMesh, optimizing per-point weights helps control mesh complexity: stronger regularization on these weights results in a simpler output mesh. Since DMesh++ lacks this mechanism, it cannot directly regulate mesh complexity during optimization. To address this limitation, we propose the \textit{Reinforce-Ball} algorithm, which reduces unnecessary faces while preserving essential geometric details.

\subsection{Local Minima of Weight Regularization}

Before delving into the details of the Reinforce-Ball algorithm, we first highlight a limitation of DMesh's per-point weight regularization. Specifically, while per-point weights are relevant for controlling mesh complexity, they alone cannot achieve adaptive resolution or produce a mesh that is both efficient and precise. Below, we explain the reasons in detail, assuming that per-point probabilities are optimized and that the \textit{Minimum-Ball} condition is employed to compute face probabilities.

In~\cref{fig:render-bias-rl}, we provide an example in a rendering scenario. A camera is placed on the left, and three different probabilistic meshes are shown on the right.

In \textbf{case (1)}, there are three points: A, B, and C. By connecting points A and B, the ground truth shape can be perfectly reconstructed, making point C redundant. Assume the optimization starts from this state, where all points have an existence probability of $1.0$. According to the \emph{Minimum-Ball} condition, the probabilities of faces $\overline{AC}$ and $\overline{BC}$ will also be $1.0$. In this scenario, if a ray from the camera intersects the mesh, the accumulated opacity will be $1.0$, representing a fully opaque surface. Consequently, the reconstruction loss will be $0.0$, as the fully opaque faces perfectly match the ground truth.

\begin{figure}
    \centering
    \includegraphics[width=\linewidth]{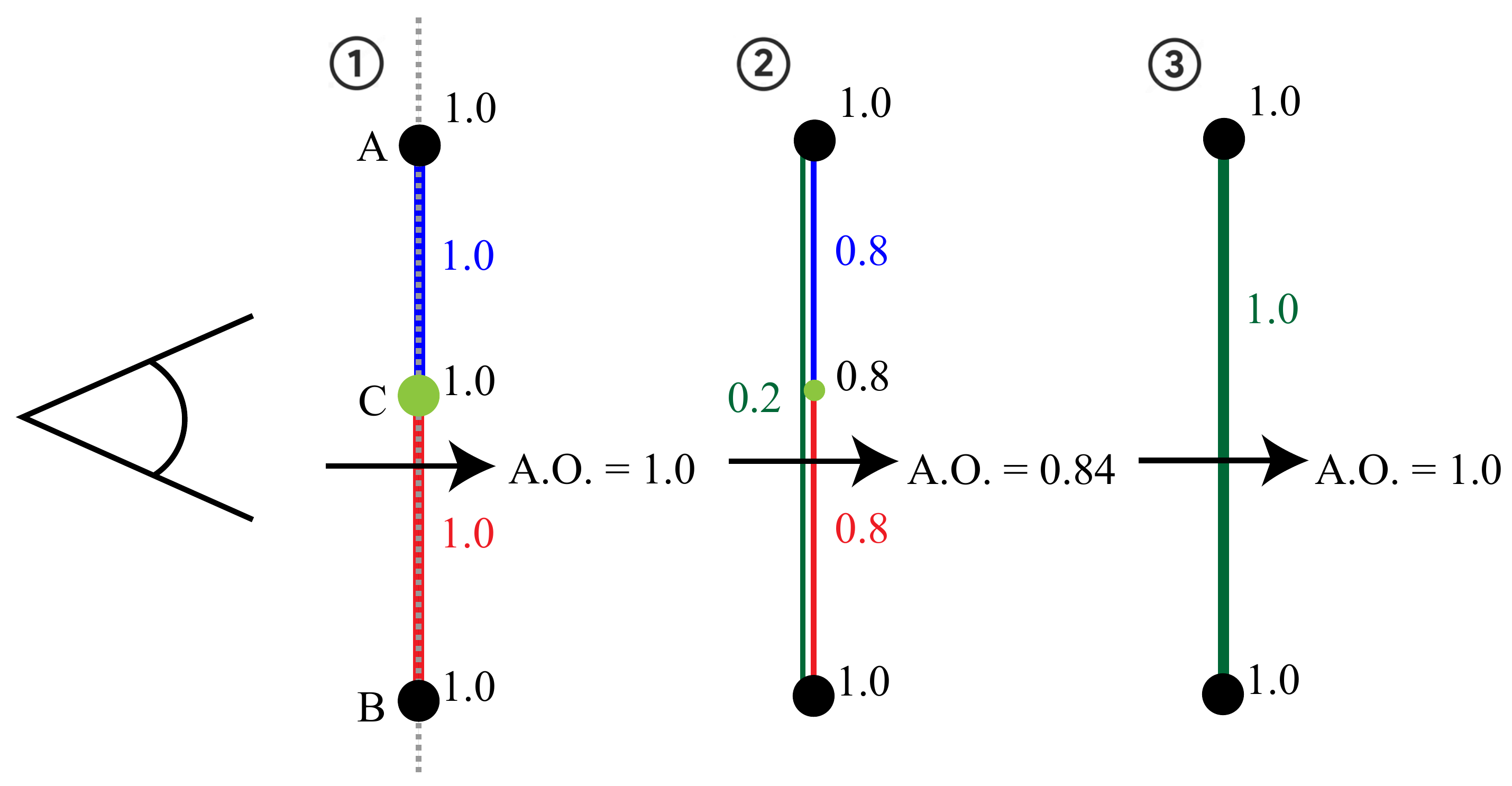}
    \caption{\textbf{Local minima of weight regularization in a rendering setting.} (1) The ground truth geometry is rendered in \textcolor{gray}{gray} dotted line. There are 3 points (A, B, \textcolor{LimeGreen}{C}), where only the end points (A, B) are necessary for fully representing the underlying shape. Every point has weight $1.0$, which is written to the next of each point. In this case, faces \textcolor{blue}{$\overline{AC}$} and \textcolor{red}{$\overline{BC}$} exist with probability $1.0$, which corresponds to their opacity. In this case, for a ray that goes through this mesh, the accumulated opacity (A.O.) becomes $1.0$, and the reconstruction loss is $0$. (2) When weight regularization reduces the weight of (redundant) \textcolor{LimeGreen}{$C$} to $0.8$, the probability of faces \textcolor{blue}{$\overline{AC}$} and \textcolor{red}{$\overline{BC}$} becomes $0.8$, and that of \textcolor{ForestGreen}{$\overline{AB}$} becomes $0.2$. However, in this case, the accumulated opacity of the same ray becomes $0.84$, which results in non-zero reconstruction loss. (3) Therefore, with a small weight regularization, we cannot remove \textcolor{LimeGreen}{C} to get this optimal mesh, which contains only \textcolor{ForestGreen}{$\overline{AB}$}, and attains $0$ reconstruction loss.}
    \label{fig:render-bias-rl}
    \vspace{-1.0em}
\end{figure}

In \textbf{case (3)}, the optimal configuration is rendered, where the redundant point C is removed. The probability of face $\overline{AB}$ becomes $1.0$, making it fully opaque. Again, the reconstruction loss is $0.0$.

In \textbf{case (2)}, an intermediate state between cases (1) and (3) is rendered. Assume that the probability of point C is reduced to $0.8$ due to regularization. Consequently, the probabilities of faces $\overline{AC}$ and $\overline{BC}$ are also reduced to $0.8$ because one of their endpoints, C, has a probability of $0.8$. Simultaneously, the probability of face $\overline{AB}$ increases from $0$ to $0.2$, as the probability of point C, which lies inside the minimum bounding ball of the face, is $0.8$. 

Now, consider a camera ray passing through $\overline{AB}$ and $\overline{BC}$ sequentially (the order does not matter due to their tight overlap). Using alpha blending, the accumulated opacity is computed as:
\begin{equation}
    \text{Accumulated Opacity: } 0.2 + (1.0 - 0.2) \cdot 0.8 = 0.84.
\end{equation}
This calculation shows that the accumulated opacity is reduced to $0.84$. 

The key issue arises from the \textbf{dependency} between the probabilities of $\overline{AB}$, $\overline{AC}$, and $\overline{BC}$. In the above formulation, the term $(1.0 - 0.2) \cdot 0.8$ represents the probability that the ray misses $\overline{AB}$ and hits $\overline{BC}$. If the probabilities of $\overline{AB}$ and $\overline{BC}$ were independent, this formulation would be correct. However, they are dependent: in fact, the probability of $\overline{BC}$ equals $1.0 - \overline{AB}$ because both depend on the probability of C. Thus, the actual accumulated opacity should be:
\begin{equation}
    0.2 + (1.0 - 0.2) \cdot 1.0 = 1.0.
\end{equation}
However, the alpha blending technique used here does not account for such dependencies, leading to a reduction in accumulated opacity. This reduction artificially increases the reconstruction loss. To minimize the loss, the optimizer increases the probability of C again, preventing convergence to the optimal case (3).

This dependency issue creates a local minimum that the previous formulation cannot overcome. This is why we propose the \emph{Reinforce-Ball} algorithm.

\begin{figure}[t]
    \centering
    \includegraphics[width=\linewidth]{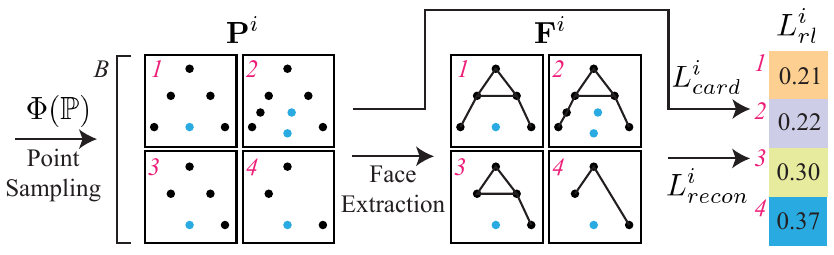}
    \caption{\textbf{Overview of \emph{Reinforce-Ball} Algorithm.} Based on per-point existence probability ($\Phi(\mathbb{P})$), we sample points for $B$ number of batches ($\mathbf{P}^i$). Here we use $B=4$, and assume we are reconstructing shape ``A''. The points with $\psi=1$ are rendered in black, while those with $\psi=0$ are rendered in blue. Then, we identify existing faces in each batch ($\mathbf{F}^{i}$) based on~\cref{eq:tessellation_ours}. With $\mathbf{P}^{i}$ and $\mathbf{F}^{i}$, we compute loss for each batch. Note that the case (1, 2) are better than (3, 4), because they reconstruct the shape better ($L^{i}_{recon}$). Also, the case (1) is better than (2), because it has less number of points ($L^{i}_{card}$). To minimize the expected loss ($\mathbb{E}[L_{rl}]$), we should maximize the probability to sample the case 1. We optimize $\Phi(\mathbb{P})$ to do that.}
    \label{fig:rl-ball}
    \vspace{-1em}
\end{figure}

\subsection{Algorithm Overview}
\label{appendix:rl-ball-overview}

In the \emph{Reinforce-Ball} algorithm, we define per-point existence probability and optimize it using stochastic optimization technique~\citep{williams1992simple}. The overview of this algorithm is given in~\cref{fig:rl-ball}.

To elaborate, for a point $p \in \mathbb{P}$, let us denote the probability of it as $\phi(p) \in [0, 1]$, and concatenation of them as $\Phi(\mathbb{P})$. Then, assuming we sample points independently, we can sample a set of points $\mathbf{P}$ from $\Phi(\mathbb{P})$ and compute its probability as follows: 
\begin{equation}
\label{eq:rl-sample-point-prob}
    P(\mathbf{P} | \Phi(\mathbb{P})) = \Pi_{p \in \mathbf{P}} \phi(p) \cdot \Pi_{p \in \mathbb{P} - \mathbf{P}} (1 - \phi(p)).
\end{equation}

Now, we sample points for $B$ batches, and denote the sample points for $i$-th batch as $\mathbf{P}^i$. Based on $\mathbf{P}^{i}$ and tessellation function in~\cref{eq:tessellation_ours}, we can find out which faces exist for the $i$-th batch. Importantly, this process does not require evaluating all possible global face combinations; instead, it focuses only on local combinations, leveraging the \emph{minimum-ball} condition in the tessellation function. We write these faces as $\mathbf{F}^{i}$, and use them for computing reconstruction loss for $i$-th batch ($L^{i}_{recon}$). We also compute ``cardinality'' loss for $i$-th batch ($L^{i}_{card}$), which is just the number of sampled points ($|\mathbf{P}^{i}|$). Then, we can compute the loss $L^{i}_{rl}$ as
\begin{equation}
    L^{i}_{rl} = L^{i}_{recon} + \epsilon_{card} \cdot L^{i}_{card},
\end{equation}

\noindent
where $\epsilon_{card}$ is a small tunable hyperparameter to adjust the weight of the cardinality loss. If we write the final loss for a set of sampled points $\mathbf{P}$ as $L_{rl}(\mathbf{P})$, we aim at minimizing the expected loss:
\begin{equation}
    \mathbb{E}_{\mathbf{P}\sim \Phi(\mathbb{P})}L_{rl}(\mathbf{P}) = \sum P(\mathbf{P} | \Phi(\mathbb{P})) \cdot L_{rl}(\mathbf{P}).
\end{equation}

\subsection{Formal Definition}

\begin{algorithm}[t]
\caption{\texttt{Reinforce-Ball}}
\label{alg:rl-ball}
\begin{algorithmic}[1]
    \State $n_{0}, n_{1} \gets$ Number of epochs and optimization steps
    \State $B \gets$ Number of batch samples
    \State $\Phi \gets$ Per-point probabilities, initialized to 0.99
    \State $i \gets 0$

    \While{$i < n_{0}$}
        \State $\mathbb{F} \gets$ \texttt{Update-Query-Faces($\mathbb{P}, \Psi$)}
        \State $B_{\mathbb{F}} \gets$ \texttt{Compute-Minimum-Ball}($\mathbb{P, F}$)

        \State $j \gets 0$

        \While{$j < n_1$}
            \State \textit{(k = 1, ..., B)}
            
            \State $\mathbf{P}^{k} \gets$ \texttt{Sample-Points}($\mathbb{P}, \Phi$)
            
            \State $\mathbf{F}^{k} \gets$ \texttt{Get-Exist-Faces}($\mathbf{P}^{k}, \mathbb{F}, B_{\mathbb{F}}$)
            
            \State $L_{rl}^{k} \gets$ \texttt{Compute-Loss($\mathbb{P}, \mathbf{P}^{k}, \mathbf{F}^{k}$)}

            \State $\frac{\partial \mathbb{E}[L_{rl}]}{\partial \Phi} \gets $ \texttt{Estimate-Gradient}($\Phi, \mathbf{P}^{k}, L_{rl}^{k}$)

            \State $\Phi \gets$ \texttt{Update-Gradient}($\Phi, \frac{\partial \mathbb{E}[L_{rl}]}{\partial \Phi}$)
        \EndWhile

        \State $\mathbb{P}, \Psi \gets$ \texttt{Get-Remaining-Points}($\mathbb{P}, \Psi, \Phi$)
    \EndWhile
\end{algorithmic}
\end{algorithm}

In~\cref{alg:rl-ball}, we formally describe the \emph{Reinforce-Ball} algorithm in detail:

\begin{itemize}
    \item \textbf{Line 1:} In the \emph{Reinforce-Ball} algorithm, we optimize per-point probabilities for $n_{0}$ epochs, with each epoch consisting of $n_1$ optimization steps. In our experiments, we set $n_0 = 10$ and $n_1 = 2000$.
    \item \textbf{Line 2:} We define the number of batches used during optimization as $B$. Increasing $B$ improves the stability of the gradient computation but also increases computational cost. In our experiments, we set $B = 1024$.
    \item \textbf{Line 3:} Initialize the per-point probability of every point to $0.99$, as all points are assumed to exist with high probability before optimization. The probabilities are not set to $1.0$ to avoid every sampled batch (Line 11) including all points, which would prevent optimization from progressing.
    \item \textbf{Lines 4-5:} Perform multiple epochs of optimization.
    \item \textbf{Line 6:} Gather the possibly existing faces ($\mathbb{F}$) based on the current point configuration and their real values. This function is the same as the one used in the Point Optimization step (Appendix~\ref{appendix:pipeline-2-pos-optim}).
    \item \textbf{Line 7:} Compute the minimum bounding ball $B_{\mathbb{F}}$ for the gathered query faces.
    \item \textbf{Lines 8-9:} Perform the optimization steps within the current epoch.
    \item \textbf{Line 10:} Consider $B$ batches, each containing a different point configuration based on the sampled points.
    \item \textbf{Line 11:} For each batch, sample points from $\mathbb{P}$ based on their probabilities $\Phi$. Each point is sampled independently, and the probability of sampling a specific batch is computed as shown in~\cref{eq:rl-sample-point-prob}. The sampled points in the $k$-th batch are denoted as $\mathbf{P}^{k}$.

    \begin{table}[t!]
    \centering
    \scalebox{0.73}{
  \begin{tabular}{l|cccc}
    \toprule
    Method \footnotesize{(hyperparameter)} & CD{\footnotesize ($\times 10^{-6}$)}$\downarrow$ & \# Verts. & \# Edges. & Time (sec)\\
    \midrule
DMesh~\cite{son2024dmesh} (0) & 1.97 & 2506 & 2245 & 30.39 \\
DMesh ($10^{-4}$) & 2.68 & 666 & 693 & 153.10 \\
DMesh ($10^{-3}$) & 12.48 & 456 & 488 & 152.37 \\
    \midrule
DMesh++ (0) & 1.82 & 2862 & 2793 & 11.33 \\
DMesh++ ($10^{-6}$) & 1.86 & 1386 & 1394 & 278.88 \\
DMesh++ ($10^{-5}$) & 2.77 & 149 & 152 & 200.05 \\
    \bottomrule
  \end{tabular}
  }
  \caption{\textbf{Quantitative ablation studies on Reinforce-Ball algorithm.} As we increase $\epsilon_{card}$ (in parenthesis) for DMesh++, we can significantly reduce the mesh complexity without losing geometric details, while DMesh cannot do the same with $\lambda_{weight}$. 
    \label{tab:rlball-font-ablation}
    \vspace{-0.5em}
    }
\end{table}

    \item \textbf{Line 12:} For each batch, determine the existing faces in $\mathbb{F}$ based on the sampled points. Specifically, a face $F$ exists if all its points are included in the sampled points and its $B_{F}$ satisfies the \emph{Minimum-Ball} condition. The existing faces in the $k$-th batch are denoted as $\mathbf{F}^{k}$.
    \item \textbf{Line 13:} For each batch, compute the loss as the sum of the reconstruction loss ($L_{recon}$) and the cardinality loss ($L_{card}$), as discussed in Appendix~\ref{appendix:rl-ball-overview}.
    \item \textbf{Line 14:} Estimate the gradient of the expected loss ($\mathbb{E}[L_{rl}]$) with respect to the per-point probabilities $\Phi$ using the log-derivative trick~\citep{williams1992simple}:
    \begin{equation}
    \label{eq:log-deriv}
        \nabla_{\Phi}\mathbb{E}_{\mathbf{P}\sim \Phi}[L_{rl}] \approx \frac{1}{B} \sum_{i=1}^{B} \nabla_{\Phi}\log P(\mathbf{P^i} | \Phi) \cdot L^{i}_{rl}.
    \end{equation}
    To reduce the variance of the gradient, we normalize $L_{rl}$ across the batch before the computation~\citep{greensmith2004variance}.
    \item \textbf{Line 15:} Update $\Phi$ using the estimated gradients.
    \item \textbf{Line 17:} After completing an epoch, discard points whose probability is below a specified threshold. In our experiments, we set the threshold to $0.5$. The remaining points are used for the next epoch. As points are removed, the query faces updated in Line 6 for the next epoch will span a larger area than in the previous epoch.
\end{itemize}

\subsection{Experimental Results}

Using the Reinforce-Ball algorithm, we can reconstruct efficient 2D meshes from point clouds that adapt to local geometry. As described in~\cref{sec:exp-2d-pc}, we conducted 2D point cloud reconstruction experiments on the font dataset.

In~\cref{tab:rlball-font-ablation}, we present quantitative ablation studies on the \emph{Reinforce-Ball} algorithm. Increasing the tunable hyperparameter $\epsilon_{card}$ (Appendix~\ref{appendix:rl-ball-overview}), which controls regularization strength, leads to a rapid reduction in vertices and edges. For instance, with $\epsilon_{card} = 10^{-5}$, edges decrease by nearly \textit{94\%} with minimal impact on reconstruction quality. DMesh~\citep{son2024dmesh} also offers a tunable parameter, $\lambda_{weight}$, for weight regularization to reduce mesh complexity. However, while edge reduction occurs, DMesh’s reconstruction quality degrades more quickly. At $\epsilon_{card} = 10^{-5}$, our method achieves a similar CD loss to DMesh with $\lambda_{weight} = 10^{-4}$ but uses about \textit{78\%} fewer edges. This advantage is also evident in~\cref{fig:rlball-font-qual}, where our \emph{Reinforce-Ball} algorithm removes redundant edges effectively and adapts the mesh to local geometry. In contrast, DMesh’s edge removal disregards local geometry, resulting in loss of detail.

\begin{figure}[t]
    \centering
    \includegraphics[width=\linewidth]{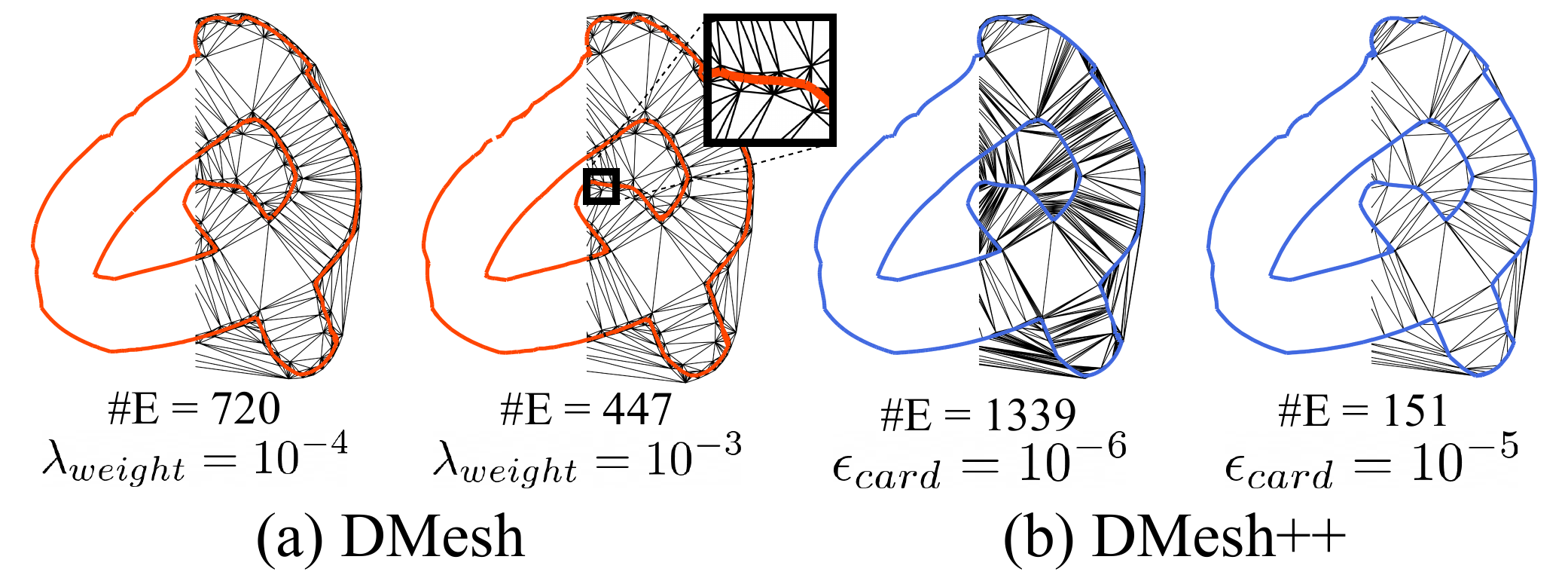}
    \caption{\textbf{Qualitative ablation studies on Reinforce-Ball algorithm (for letter `Q').} We render ``imaginary'' (black) part and ``real part'' (red, blue) together. } 
    \label{fig:rlball-font-qual}
    \vspace{-0.5em}
\end{figure}

Likewise, we successfully highlighted the limitation of DMesh's weight regularization and demonstrated that the \emph{Reinforce-Ball} algorithm can eliminate redundant mesh faces within the DMesh++ framework without sacrificing geometric details. However, since the method is not yet easily extensible to 3D and incurs high computational costs, we include these results in the Appendix.